\documentclass[onefignum,onetabnum]{siamonline181217}

\usepackage[utf8]{inputenc}
\usepackage{amssymb,bbm,mathrsfs,bm,csquotes,mathtools,bookmark} 

\usepackage{enumitem}
\setlist[enumerate]{leftmargin=.5in,itemsep=7pt,topsep=7pt}
\setlist[itemize]{leftmargin=.5in}

\ifpdf
\hypersetup{
  pdftitle={Deep Learning--Based ERM for Black--Scholes PDEs},
  pdfauthor={J. Berner, P. Grohs, and A. Jentzen}
}
\fi

\newsiamremark{setting}{Setting}

\providecommand{\N}{{\ensuremath{\mathbb{N}}}}
\providecommand{\R}{{\ensuremath{\mathbb{R}}}}
\providecommand{\cH}{{\ensuremath{\mathcal{H}}}}
\providecommand{\cN}{{\ensuremath{\mathcal{N}}}}
\providecommand{\cont}{{\ensuremath{\mathcal{C}}}}
\providecommand{\leb}{{\ensuremath{\mathcal{L}}}}
\providecommand{\E}{{\ensuremath{\mathbbm{E}}}}
\renewcommand{\P}{\mathbbm{P}}
\newcommand{\PX}{{\ensuremath{\P_{X_d}}}}
\providecommand{\cE}{{\ensuremath{\mathcal{E}}}}
\providecommand{\B}{{\ensuremath{\mathcal{B}}}}
\newcommand{\G}{{\ensuremath{\mathcal{G}}}}
\newcommand{\bA}{{\ensuremath{\mathbf{A}}}}
\newcommand{\ba}{{\ensuremath{\mathbf{a}}}}
\newcommand{\peta}{{\ensuremath{\bm{\eta}}}}
\newcommand{\ptheta}{{\ensuremath{\bm{\theta}}}}
\DeclareMathOperator{\ReLU}{ReLU}
\newcommand{\F}{{\ensuremath{\mathcal{F}}}}
\newcommand{\cA}{{\ensuremath{\mathcal{A}}}}
\newcommand{\p}{{\ensuremath{\mathcal{P}}}}
\newcommand{\cC}{{\ensuremath{\operatorname{clip}}}}
\newcommand{\cL}{{\ensuremath{L}}}
\newcommand{\sz}{{\ensuremath{P}}}
\newcommand{\cov}{{\operatorname{Cov}}}
\providecommand{\1}{{\ensuremath{\mathbbm{1}}}}
\newcommand{\cO}{{\ensuremath{\mathcal{O}}}}
\newcommand{\bd}{{\ensuremath{D}}}
\newcommand{\eu}{{\ensuremath{2}}}
\newcommand{\rr}{{\ensuremath{r}}}
\newcommand{\f}{{\ensuremath{f}}}
\newcommand{\g}{{\ensuremath{g}}}
\newcommand{\neps}{{\ensuremath{\delta}}}

\providecommand{\zero}{{\ensuremath{0}}}

\headers{Deep Learning--Based ERM for Black--Scholes PDEs}{J. Berner, P. Grohs, and A. Jentzen}

\title{Analysis of the Generalization Error: Empirical Risk Minimization over\\Deep Artificial Neural Networks Overcomes the Curse of Dimensionality in the\\Numerical Approximation of Black--Scholes Partial Differential Equations}
    
\author{
Julius Berner\thanks{Faculty of Mathematics, University of Vienna, Austria (\email{julius.berner@univie.ac.at}).} 
\and
Philipp Grohs\thanks{Faculty of Mathematics and Research Platform DataScience@UniVienna, University of Vienna, Austria (\email{philipp.grohs@univie.ac.at}).}
\and
Arnulf Jentzen\thanks{Department of Mathematics, ETH Z\"urich, Switzerland, and Faculty of Mathematics and Computer Science, University of M\"unster, Germany (\email{ajentzen@uni-muenster.de}).}
\funding{This work has been funded by the Deutsche Forschungsgemeinschaft (DFG, German Research Foundation) under Germany's Excellence Strategy EXC 2044-390685587, Mathematics M\"unster: Dynamics-Geometry-Structure and by the Austrian Science Fund (FWF) under grant I3403-N32.}}

\begin{document}

\maketitle

\begin{abstract}
The development of new classification and regression algorithms based on empirical risk minimization (ERM) over deep neural network hypothesis classes, coined deep learning, revolutionized the area of artificial intelligence, machine learning, and data analysis. In particular, these methods have been applied to the numerical solution of high-dimensional partial differential equations with great success. Recent simulations indicate that deep learning--based algorithms are capable of overcoming the curse of dimensionality for the numerical solution of Kolmogorov equations, which are widely used in models from engineering, finance, and the natural sciences. The present paper considers under which conditions ERM over a deep neural network hypothesis class approximates the solution of a $d$-dimensional Kolmogorov equation with affine drift and diffusion coefficients and typical initial values arising from problems in computational finance up to error $\varepsilon$. We establish that, with high probability over draws of training samples, such an approximation can be achieved with both the size of the hypothesis class and the number of training samples scaling only polynomially in $d$ and $\varepsilon^{-1}$. It can be concluded that ERM over deep neural network hypothesis classes overcomes the curse of dimensionality for the numerical solution of linear Kolmogorov equations with affine coefficients.
\end{abstract}



\section{Introduction}
In this introductory section we want to present and motivate our problem, provide the reader with background knowledge and references to previous research on the topic, and outline the important steps, as well as possible extensions, of our contribution. 
\subsection{Problem statement}
Suppose we need to numerically approximate the end value\footnote{We write the subscript $d$ as we are interested in approximation rates w.r.t.\@ to the dimension $d\in\N$.} 
$F_d(T,\cdot)$ at time $T\in(0,\infty)$
of the solution $F_d\in \cont ([0,T]\times \R^d,\R)$ of a \emph{linear Kolmogorov equation} which for an initial value $\varphi_d\in \cont(\R^d,\R)$, drift coefficient 
$\mu_d\in \cont (\R^d,\R^d)$, and diffusion coefficient $\sigma_d\in \cont (\R^d,\R^{d\times d})$
is defined 
as\footnote{For $x\in\R^d$, $y\in\R^d$ we denote by $x\cdot y:=\sum_{i=1}^d x_i y_i$ the standard scalar product of $x$ and $y$.}
\begin{equation}\label{eq:kolmointro}
\begin{cases}
\frac{\partial F_d}{\partial t}(t,x)=\frac12\mathrm{Trace}\big(\sigma_d(x)[\sigma_d(x)]^* (\mathrm{Hess}_x F_d)(t,x)\big) 
+ \mu_d(x) \cdot (\nabla_x F_d)(t,x) 
\\ F_d(0,x)=\varphi_d(x) 
\end{cases}
\end{equation}
for every $(t,x)\in [0,T]\times \R^d$.
Important special cases include the heat equation or the Black--Scholes equation from computational finance. For the latter partial differential equation (PDE) typically the coefficients $\sigma_d$ and $\mu_d$ are affine and the initial value $\varphi_d$ can be represented as a composition of minima, maxima, and linear combinations such as
\begin{equation}\label{eq:EuroPut}\varphi_d(x)=\min\big\{\max\left\{\bd- c_{d}\cdot x,0\right\},\bd\big\}\end{equation} 
with suitable coefficients $\bd\in (0,\infty)$, $c_{d}\in\R^d$
in the case of a European put option pricing problem. 
It is well known that standard numerical methods for solving PDEs, in particular those based on a discretization of the domain, suffer from the curse of dimensionality, meaning that their computational complexity grows exponentially in the dimension $d$~\cite{ames2014numerical,seydel2012tools}.

If the goal is simply to evaluate $F_d(T,\cdot)$ \emph{at a single value} $\xi\in \R^d$, then under suitable assumptions Monte Carlo sampling methods are capable of overcoming the curse of dimensionality. These methods are based on the integral representation (Feynman--Kac formula)
\begin{equation}\label{eq:FC}
F_d(T,\xi) = \mathbb{E}\big[\varphi_d(S_T^\xi)\big],
\end{equation}
where $(S_t^\xi)_{t\in[0,T]}$ is a stochastic process satisfying the stochastic differential equation (SDE)
\begin{equation*}
dS_t^\xi=\sigma_d(S_t^\xi)dB^d_t + \mu_d(S_t^\xi)dt\quad \text{and}\quad S^{\xi}_0=\xi
\end{equation*}
for a $d$-dimensional
$(\G_t)$-Brownian
motion $B^d$ on some filtered probability space 
$(\Omega,\G,\P,(\G_t))$.
The evaluation of $F_d(T,\xi)$ can then be computed by
approximating the expectation in \eqref{eq:FC} by Monte Carlo integration, that is, by simulating i.i.d.\@ samples $(S^{(i)})_{i=1}^{m}$ drawn from the distribution of $S_T^{\xi}$ and by approximating $F_d(T,\xi)$ with the empirical average $\frac{1}{m}\sum_{i=1}^{m} \varphi_d(S^{(i)})$.
It is a standard result that the number of samples $m$ needed to obtain a desired accuracy $\varepsilon$ depends only polynomially on the dimension $d$ and $\varepsilon^{-1}$~\cite{graham2013stochastic}.

If the goal, however, is to approximate $F_d(T,\cdot)$ not only at a single value but, for example, on a full hypercube $[u,v]^d$, there has been no known method which does not suffer from the curse of dimensionality. In particular, it has been completely out of range to provably approximate $F_d(T,\cdot)$ on $[u,v]^d$ in high dimensions, say, $d\gg 100$.

The present paper introduces and analyzes deep learning--based algorithms for the numerical approximation of $F_d(T,\cdot)$ on a full hypercube $[u,v]^d$. We will prove that the resulting algorithms overcome the curse of dimensionality and can consequently be efficiently applied even in high dimensions. Our proofs will be based on tools from statistical learning theory and the
following key properties of linear Kolmogorov equations:
\begin{enumerate}[label=P.\arabic*]
\item The fact that one can reformulate~\eqref{eq:kolmointro} as a learning problem (see Lemma~\ref{lem:kolm_reg}).
\item \label{prop:approx} The fact that typical initial values arising from problems in computational finance, such as, for example,~\eqref{eq:EuroPut}, are either exactly representable as neural networks with ReLU activation function (ReLU networks) or can be approximated by such neural networks without incurring the curse of dimensionality (see \cite[Section 4]{grohs2018approx}).
\item The fact that Property \ref{prop:approx} is preserved under the evolution of linear Kolmogorov equations~\eqref{eq:kolmointro} with affine diffusion and drift coefficients, which implies that
$F_d(T,\cdot)$ can be approximated by ReLU networks without incurring the curse of dimensionality (see Theorem \ref{thm:mainapprox}).
\end{enumerate}

\subsection{Deep learning and statistical learning theory}\label{sec:DLERM}
In their most basic incarnation, deep learning--based algorithms start with training data 
$$((X_d^{(i)},Y_d^{(i)}))_{i=1}^{m}\colon \Omega\to ([u,v]^d\times [-\bd,\bd])^m.$$
To give a concrete example, 
$X_d^{(i)}$
may consist of different 
$28\times 28$
pixel grayscale images of handwritten digits and $Y_d^{(i)}$ 
may consist of corresponding probabilities describing the likelihood of a certain digit to be shown in image $X_d^{(i)}$ \cite{mnist}. 
The goal is then to find a functional relation between images and labels and use it for predictive purposes 
on unseen images. 

Empirical risk minimization (ERM) attempts to solve this prediction problem by minimizing the empirical risk
\begin{equation}\label{eq:emprisk}
f\mapsto \widehat{\cE}_{d,m}(f):=\tfrac{1}{m}\sum_{i=1}^{m}\big(f(X_d^{(i)})-Y_d^{(i)}\big)^2 
\end{equation}
over a compact\footnote{Note that we equip $\cont([u,v]^d,\R)$ with the uniform norm $\|{\cdot}\|_{\leb^{\infty}}$ which for $f\in \cont([u,v]^d,\R)$ is given by $\|f\|_{\leb^{\infty}}=\|f\|_{\leb^{\infty}([u,v]^d)}:= \max_{x\in[u,v]^d} |f(x)|$.} hypothesis class $\cH\subseteq \cont([u,v]^d,\R)$, resulting in a predictor
\begin{equation*} 
\widehat f_{d,m,\cH}\in \mathop{\mathrm{argmin}}_{f\in \cH}\widehat{\cE}_{d,m}(f)
\end{equation*}
that is hoped to provide a good approximation of the desired functional relation in the training data. In deep learning, these hypothesis classes consist of deep neural networks with fixed activation function $\rho\in \cont(\R,\R)$, parameter bound $R\in(0,\infty)$, and architecture\footnote{Typically one calls a neural network \enquote{deep} if the architecture satisfies $L> 2$.} $\ba = (a_0,a_1,\dots,a_{L})\in\N^{L+1}$, where $a_0=d$ and $a_L=1$. 
We define the corresponding set of neural network parametrizations
$$
\p_{\ba,R} := \bigtimes_{l=1}^{L} \big( [-R,R]^{a_{l}\times a_{l-1}} \times [-R,R]^{a_{l}} \big),
$$
and for a parametrization $\ptheta =((W_l,B_l))_{l=1}^L\in\p_{\ba,R}$ we define its realization function\footnote{If there is no possibility of ambiguity, we use the term \enquote{neural network} interchangeably for the parametrization and the realization function. However, note that $\ptheta\in\p_{\ba,R}$ uniquely induces $\F_\rho(\ptheta)\in \cont(\R^d,\R)$, while in general there can be multiple nontrivially different parametrizations with the same realization function; see~\cite{bernerdeg2019}.} 
\begin{equation*}
\F_{\rho}(\ptheta):=\cA_{W_L,B_L} \circ \rho_* \circ \cA_{W_{L-1},B_{L-1}} \circ \rho_* \circ 
\dots \circ \rho_* \circ \cA_{W_1,B_1} \in \cont(\R^d,\R),
\end{equation*}
where $\cA_{W,B}(x):=Wx+B$ and
$\rho_*(x)=(\rho(x_i))_{i=1}^n$, i.e., $\rho$ is applied componentwise. Then, neural network hypothesis classes are typically of the form
\begin{equation}\label{eq:networkclass}
\cN^{u,v}_{\rho,\ba,R}:=\left\{\left([u,v]^{d}\ni x\mapsto \F_{\rho}(\ptheta)(x)\right):\ \ptheta \in \p_{\ba,R}
\right\}.
\end{equation}
Despite the great practical success of the \enquote{deep learning paradigm,} for generic real world training data it is far out of reach to specify network architectures 
which guarantee a desired performance on unseen data; see also~\cite{zhang2016}.

This type of problem can be theoretically studied using tools developed within the field of statistical learning theory. There it is typically postulated that
$((X_d^{(i)},Y_d^{(i)}))_{i=1}^{m}$ are i.i.d.\@ samples drawn from the distribution of some (unknown)\@ data $(X_d,Y_d)$ and that the optimal functional relation between $X_d$ and $Y_d$ is given by the regression function
$$f^*_{d}:\left\{\begin{array}{ccc}[u,v]^d &\to &\R \\ x&\mapsto & \mathbb{E}\left[Y_d \big| X_d=x\right], \end{array}\right. $$ 
which minimizes the risk
$
f\mapsto \cE_d(f):=\E\left[\big(f(X_d)-Y_d\big)^2 \right]
$
(see Lemma~\ref{lem:biasvariance}).
The minimization of functionals of the form $\cE_d$ is commonly referred to as a
\vspace{4pt}
\begin{displayquote}
\emph{statistical learning problem with data $(X_d,Y_d)$ and quadratic loss function.}
\end{displayquote}
\vspace{4pt}
Under strong regularity assumptions on the regression function $f^*_{d}$ and the distribution of $(X_d,Y_d)$ it is possible to obtain bounds on the sample size $m$ and the \enquote{complexity} of the hypothesis class $\cH$
in order to guarantee, with high probability, an error \begin{equation}\label{eq:SLerror}\E\left[\big(\widehat f_{d,m,\cH}(X_d)-f^*_{d}(X_d)\big)^2\right]\le \varepsilon;\end{equation} see, for example,~\cite{anthony2009,bartlett2005local,cucker2002,cucker_zhou_2007,gyorfi2006distribution,koltchinskii2011introduction,massart2007concentration,van2000applications}. In the above, the regularity of the regression function $f^*_d$ quantifies how well $f^*_d$ can be approximated by the hypothesis class $\cH$. 

In the case of the neural network hypothesis classes $\cH=\cN^{u,v}_{\rho,\ba,R}$ these regularity assumptions are met if $f^*_d$ satisfies certain smoothness assumptions; see~\cite{bolcskei2017optimal,Burger2001235,Funahashi1989183,perekrestenko2018universal,petersen2017optimal,ShaCC2015provableAppDNN,yarotsky2017error}.
Moreover, the complexity of $\cN^{u,v}_{\rho,\ba,R}$ can mainly be described by the size of the neural network parametrizations, i.e., the number of network parameters 
\begin{equation*}
    \textstyle \sz(\ba):=\sum_{l=1}^L a_la_{l-1}+a_l.
\end{equation*}
We will show in Subsection~\ref{sub:kol_learn} below that our specific deep learning--based method for numerically solving Kolmogorov equations allows us to rigorously apply tools from statistical learning theory, as we can overcome the following potential problems: 
\begin{enumerate}[label=R.\arabic*]
\item \label{it:prob1} The crucial assumption that the training data consists of i.i.d.\@ samples drawn from an underlying probability distribution is usually debatable or at least hard to verify. 
\item Even if this assumption were satisfied, the underlying distribution of $(X_d,Y_d)$ is typically unknown. Thus, it is hard to ensure a priori the regularity assumptions on $f^*_{d}$, which are needed to apply tools from statistical learning theory.
\item Since the distribution of $X_d$ is typically unknown, it is not clear how the 
quantity $\E\big[\big(\widehat f_{d,m,\cH}(X_d)-f^*_{d}(X_d)\big)^2\big]$ of~\eqref{eq:SLerror} can be interpreted. 
\item \label{it:prob4} Most of the classical techniques operate in an asymptotic regime where the number of training samples $m$ exceeds the network size $\sz(\ba)$. However, in many applications the number of training samples is fixed, and it is not possible to generate more training data at will. 
\end{enumerate}
\subsection{Kolmogorov equations as learning problem} \label{sub:kol_learn}
We will reformulate the numerical approximation of $F_d(T,\cdot)$ on $[u,v]^d$
as a statistical learning problem and demonstrate that in this specific case none of the aforementioned Problems~\ref{it:prob1}--\ref{it:prob4} appears. 
Let $X_d\sim \mathcal{U}([u,v]^d)$ be
uniformly distributed on $[u,v]^d$ and define $Y_d:=\varphi_d(S^{X_d}_T)$, where $(S_t^{X_d})_{t\in[0,T]}$ is a stochastic process satisfying the SDE
\begin{equation} \label{eq:sde}
dS_t^{X_d}=\sigma_d(S_t^{X_d})dB^d_t + \mu_d(S_t^{X_d})dt\quad \text{and}\quad S^{X_d}_0=X_d.
\end{equation}
Under suitable conditions it then follows from the Feynman--Kac formula~\eqref{eq:FC} that $F_d(T,\cdot)$ is the minimizer of the risk functional $\cE_d$,
that is, $f^*_{d} (x) = F_d(T,x)$
for a.e.\@ $x\in[u,v]^d$; see Lemma~\ref{lem:kolm_reg}. 
As outlined in Subsection \ref{sec:DLERM}, we thus have that
\vspace{4pt}
\begin{displayquote}
\emph{the end value of the Kolmogorov equation $F_d(T,\cdot)$ is the solution to the statistical learning problem with data $(X_d,\varphi_d(S^{X_d}_T))$ and quadratic loss function}.
\end{displayquote}
\vspace{4pt}
A natural next step is to apply the deep learning paradigm, that is, for $m$ 
i.i.d.\@ samples 
$((X_d^{(i)},Y_d^{(i)}))_{i=1}^m$ drawn from the distribution of $(X_d,Y_d)$ to minimize the empirical risk~\eqref{eq:emprisk} over a hypothesis class 
of neural networks $\cN^{u,v}_{\rho,\ba,R}$. 
    
In~\cite{beckbecker2018} this idea has been implemented with suitable classes of 
deep neural networks of a given architecture as hypothesis class $\cH$.
In extensive numerical simulations it was observed that the proposed algorithm
is efficient even in very high dimensions, suggesting that it does not suffer from the curse of dimensionality. 
Similar conclusions can be found in related work~\cite{beck2019machine,becker2019solving,EHanJentzen2017a,EYu2017,FujiiTakahashiTakahashi2017,EHanJentzen2017b,Labordere2017,SirignanoSpiliopoulos2017}, which covers topics ranging from American option pricing problems to fully nonlinear PDEs. Note that in~\cite{SirignanoSpiliopoulos2017} a nonquantitative analysis of the approximation error is given; all the other works are purely empirical.
To the best of our knowledge, this is the first joint quantitative analysis of approximation and generalization error confirming the efficiency of deep learning--based methods applied to the numerical solution of high-dimensional PDEs.

Two main parameters influence the complexity of the algorithm described above: the number $P(\ba_{d,\varepsilon})$ of network parameters that need to be optimized as well as the number of training samples $m_{d,\varepsilon}$ needed to guarantee that, with high probability, the estimate 
\begin{equation}\label{eq:error}
\tfrac{1}{(v-u)^d}\big\|\widehat f_{d,m_{d,\varepsilon},\cH_{d,\varepsilon}} - F_d(T,\cdot)\big\|_{\leb^2([u,v]^d)}^2 =\mathbb{E}\left[\big(\widehat f_{d,m_{d,\varepsilon},\cH_{d,\varepsilon}}(X_d) - F_d(T,X_d)\big)^2\right]\le \varepsilon
\end{equation}
holds true. We are interested in the scaling of $m_{d,\varepsilon}$ and $P(\ba_{d,\varepsilon})$ with respect to the precision $\varepsilon$ and dimension $d$. 

Observe that the data distribution $(X_d,\varphi_d(S^{X_d}_T))$ is now explicitly known and i.i.d.\@ samples of this distribution can be simulated as needed ($X_d$ is uniformly distributed and can be simulated using a suitable random number generator, and $S^{X_d}_T$ can be simulated by any numerical solver for the SDE~\eqref{eq:sde}; see~\cite{graham2013stochastic}). Moreover, the uniform distribution of the input data $X_d$ gives rise to typical $\leb^2$-error estimates~\eqref{eq:error} and, in the case of affine coefficients and suitable initial values, we can establish bounds on how well the regression function can be approximated by neural network hypothesis classes; see Theorem~\ref{thm:mainapprox}.
In particular, contrary to conventional learning problems, in the statistical learning problem that arises from our reformulation of the Kolmogorov equation, none of the Problems~\ref{it:prob1}--\ref{it:prob4}\@ described in Subsection \ref{sec:DLERM} occurs. We will therefore be able to rigorously invoke tools from statistical learning theory to obtain bounds on the quantities $m_{d,\varepsilon}$ and $P(\ba_{d,\varepsilon})$ above.

\subsection{Contribution}
We show that whenever $(\sigma_d)_{d\in\N}$ and $(\mu_d)_{d\in\N}$ are affine functions (this includes the important case of the Black--Scholes equation in option pricing) and the initial values $(\varphi_d)_{d\in \N}$ can be approximated by deep neural networks without the curse of dimensionality (this can easily shown to be true for a large number of relevant options such as basket call, basket put, call on max, and call on min options), there exists a polynomial $p\colon\R^2\to\R$ such that for every $d\in\N$, $\varepsilon\in(0,1)$ it holds that
$$
\max\{m_{d,\varepsilon},P(\ba_{d,\varepsilon})\}
\le p(\varepsilon^{-1},d);
$$
see Corollary~\ref{cor:kolm_main}.
We conclude that the aforementioned deep learning--based algorithm \emph{does not suffer from the curse of dimensionality.}

We briefly describe our proof strategy for bounding the error between the empirical risk minimizer and the end value of the Kolmogorov equation 
\begin{equation*}
\tfrac{1}{(v-u)^d}\big\|\widehat f_{d,m,\cH} - F_d(T,\cdot)\big\|_{\leb^2([u,v]^d)}^2
 \end{equation*}
as in \eqref{eq:error}.
By the so-called bias-variance decomposition we can represent this error as the sum of a generalization error and an approximation error, i.e.,
\begin{equation*}
\underbrace{\vphantom{\big\|_{\leb^2([u,v]^d)}}\cE_d(\widehat f_{d,m,\cH})-\min_{f\in\cH}\cE_d(f)}_{\text{generalization error}} + \underbrace{\min_{f\in\cH} \tfrac{1}{(v-u)^d} \big\|f - F_d(T,\cdot)\big\|_{\leb^2([u,v]^d)}^2}_{\text{approximation error}};
\end{equation*}
see Lemma~\ref{lem:biasvariance}. 
Bounds on the size of neural networks with ReLU activation function $$\rho(x)=\ReLU(x):=\max\{x,0\}$$ needed to approximate $F_d(T,\cdot)$ up to a desired error have been analyzed in~\cite{grohs2018approx}. In Theorem~\ref{thm:mainapprox} we notably extend these approximation results by proving corresponding bounds on the parameter magnitudes. This is done by analyzing the special structure of the solution to the SDE~\eqref{eq:sde} in the case of affine coefficients $\sigma_d$ and $\mu_d$, employing the Feynman--Kac formula~\eqref{eq:FC}, and constructing a neural network simulating Monte Carlo sampling.
Together with the results of~\cite{grohs2018approx}, we achieve that neural network hypothesis classes with ReLU activation function are capable of approximating the end values $(F_d(T,\cdot))_{d\in\N}$ without incurring the curse of dimensionality whenever the same is true for the initial values $(\varphi_d)_{d\in\N}$.

We then leverage the approximation results as well as tools from~\cite{anthony2009,cucker2002,cucker_zhou_2007} to obtain probabilistic estimates of the generalization error; see Theorem~\ref{thm:kolm_nngen}. These tools require bounds on the covering numbers\footnote{The covering number $\cov(\cH,\rr)$ is the minimal number of balls of radius $\rr$ 
covering $\cH$; see Setting~\ref{set:cover}.} $\cov(\cH,\rr)$
of hypothesis classes consisting of neural networks. To this end, we compute the Lipschitz constant of the operator $\F_\rho$ which maps neural network parametrizations to the corresponding realization functions; see Theorem~\ref{thm:nn_lip}. Using a standard result on the covering number of balls in a Euclidean space we obtain that\footnote{For $p \ge1$, a finite index set $I$, and $M\in\R^I$ we define $\|M\|_\infty:=\max_{i\in I} |M_{i}|$ and $\|M\|_p:=\big(\sum_{i\in I} |M_i|^p\big)^{1/p}$.}
\begin{equation*}
 \ln \cov\big(\cN^{u,v}_{\ReLU,\ba,R},\rr\big)\le \sz(\ba)\Big[ \ln\Big( \frac{4L^2 \max\{1,|u|,|v|\} }{\rr}\Big) + L \ln\big( R \|\ba\|_\infty \big)\Big];
\end{equation*}
see Proposition~\ref{prop:cov_nn}. In conjunction with Hoeffding's inequality this allows us to uniformly (over the hypothesis class of neural networks) bound the error between the risk and the empirical risk. 
However, this requires that the regression function $f^*_{d}$ as well as all functions in $\cH$ are uniformly bounded. To that end we assume the initial value $\varphi_d$ to be bounded, which by the Feynman--Kac formula~\eqref{eq:FC} implies that also the function $f^*_{d}=F_d(T,\cdot)$ is bounded. Moreover, we introduce hypothesis classes of \enquote{clipped} neural networks 
\begin{equation*}
\cN^{u,v}_{\rho,\ba,R,\bd}:=\left\{\cC_{\bd}\circ g \colon g\in \cN^{u,v}_{\rho,\ba,R} \right\},
\end{equation*}
where $\cC_\bd:=\min\{|{\cdot}|,\bd\}\mathop{\mathrm{sgn}}(\cdot)$ denotes a clipping function with clipping amplitude $\bd$. This can be interpreted as incorporating the prior knowledge about the boundedness of the regression function into our hypothesis class.
In Appendix~\ref{app:nn} we show that the clipping function can be represented as a small ReLU network so that clipped ReLU networks are in fact standard neural networks.

Note that there exist different concepts and known results 
in order to bound the generalization error (see, for instance,~\cite{anthony2009,arora2018stronger,BartlettFT17,Bartlett2017Nearly-tightNetworks,Golowich17,neyshabur2017exploring}). The present paper intends to stress the interplay between the approximation and generalization error and gives a complete proof in order to rigorously show the absence of the curse of dimensionality for our particular problem. 

We are now ready to formulate a first specific result of this paper as an appetizer. It demonstrates that deep learning--based ERM succeeds in solving 
the option pricing problem for European put options without incurring the curse of dimensionality.
\begin{theorem}[pricing of options without curse of dimensionality]\label{thm:european}
Let $T,K,\bd \in[1,\infty)$, $u\in \R$, and $v\in (u,\infty)$,
and let $(\Omega,\G,\P,(\G_t))$ be a filtered probability space. For every dimension $d\in\N$ let $c_{d}\in [-D,D]^d$,
let the initial value $\varphi_d\in \cont(\R^d,\R)$ satisfy for every $x\in\R^d$ that
$$\varphi_d(x)=\min\big\{\max\left\{\bd- c_{d}\cdot x,0\right\},\bd\big\},$$
let the drift and diffusion coefficients $\mu_d\in \cont(\R^d,\R^d)$, $\sigma_d\in \cont(\R^d,\R^{d\times d})$ be affine functions satisfying for every $x\in\R^d$ that
$$
\|\sigma_d(x)\|_\eu+\|\mu_d(x)\|_{\eu}\le K(1 + \|x\|_{\eu}),
$$
and let
$F_d\in \cont([0,T]\times \R^d,\R)$ 
be the unique at most polynomially growing viscosity solution\footnote{We refer the interested reader to~\cite{HairerHutzenthalerJentzen_LossOfRegularity2015} for the definition and properties of viscosity solutions.} of the corresponding $d$-dimensional Kolmogorov equation
\begin{equation*}
\begin{cases}
\frac{\partial F_d}{\partial t}(t,x)=\frac12\mathrm{Trace}\big(\sigma_d(x)[\sigma_d(x)]^* (\mathrm{Hess}_x F_d)(t,x)\big) 
+ \mu_d(x) \cdot (\nabla_x F_d)(t,x) 
\\ F_d(0,x)=\varphi_d(x).
\end{cases}
\end{equation*}
For every $d\in\N$ let the input data $X_d\sim\mathcal{U}([u,v]^d)$ be uniformly distributed on $[u,v]^d$ and 
$\G_0$-measurable, let $B^d$ be a
$d$-dimensional $(\G_t)$-Brownian motion, let 
$(S_t^{X_d})_{t\in[0,T]}$ be a $(\G_t)$-adapted stochastic process with continuous sample paths satisfying the SDE
\begin{equation*} 
dS_t^{X_d}=\sigma_d(S_t^{X_d})dB_t^d + \mu_d(S_t^{X_d})dt\quad \text{and}\quad S^{{X_d}}_0={X_d}
\end{equation*}
$\P$-a.s.\@ for every $t\in [0,T]$,
define the label
$Y_d:= \varphi_d(S_T^{X_d})$,
and let $((X_d^{(i)},Y_d^{(i)}))_{i\in\N}$ 
be i.i.d.\@ random variables (training data) with 
$(X_d^{(1)},Y_d^{(1)})\sim\left(X_d,Y_d\right)$. Then there exists a constant
$C\in(0,\infty)$ such that the following holds: For every $d,m\in \N$, $\varepsilon,\varrho\in (0,1)$ with 
$$  m\ge C d \varepsilon^{-4} \big(1+\ln(d\varepsilon^{-1}\varrho^{-1})\big) \quad \text{(number of samples)}$$
there exist
$\ba=(d,a_1,a_2,1)\in \N^4$ and $R\in [1,\infty)$ such that it holds that
\begin{enumerate}[label=(\roman*)]
\item $
\P\left[\tfrac{1}{(v-u)^d}\big\|\widehat f_{d,m,\cH}-F_d(T,\cdot)\big\|_{\leb^2([u,v]^d)}^2 \le \varepsilon\right] \geq 1-\varrho
$,
\item $\sz(\ba)\le C d\varepsilon^{-2}$ (number of parameters),
\item $R \le  Cd^{3/2} \varepsilon^{-1}$ (parameter bound), and
\item $\max\{ a_1,a_2 \} \le  Cd \varepsilon^{-1}$ (size of the architecture),
\end{enumerate}
where $\widehat f_{d,m,\cH}\in \mathop{\mathrm{argmin}}_{f\in \cH}
\frac{1}{m}\sum_{i=1}^{m}\big(f(X_d^{(i)})-Y_d^{(i)}\big)^2 $ is a measurable empirical risk minimizer in the corresponding hypothesis class of clipped ReLU networks $\cH:=\cN^{u,v}_{\ReLU ,\ba,R,\bd}$.
\end{theorem}
A proof will be given in Subsection~\ref{sec:european}. In a more general context, Theorem~\ref{thm:kolm_nngen} states that a result analogous to Theorem \ref{thm:european} holds true whenever the initial values $(\varphi_d)_{d\in \N}$ can be approximated by ReLU networks without the curse of dimensionality.

Note that our analysis does not consider the computational cost of solving the nonsmooth, nonconvex ERM problem~\eqref{eq:emprisk}. This is typically achieved by stochastic first order optimization methods whose theoretical analysis is still an open problem. While there are many interesting approaches to the latter question,
they tend to require very strong assumptions (e.g., (almost) linearity, convexity, extreme overparametrization, or inverse stability of $\F_\rho$~\cite{Allen-Zhu2018AOver-Parameterization,bernerdeg2019,choromanska2015loss,Du2018GradientNetworks,kawaguchi2016deep,li2018learning,li2017convergence,mei2018mean,Shamir2013}), which we want to avoid in our analysis.

\subsection{Extensions}
Our results in Section~\ref{sec:pde}, in conjunction with the results of \cite{grohs2018approx}, can be applied to prove the absence of the curse of dimensionality in the pricing of (capped) basket call, basket put, call on max, and call on min options. 
Moreover, the results of Section~\ref{sec:mainresult} hold for a general statistical learning problem within Setting~\ref{set:MathLearn}. That is, ReLU network approximation results for the regression function translate directly into generalization results without incurring the curse of dimensionality. If suitable learning problems can be established, this work can extend various neural network approximation results for PDEs (see, e.g.,~\cite{grohs2019space,kutyniok2019theoretical,schwab2019deep}) to also consider the generalization error and get one step closer to a full error analysis.
For instance, there are stronger approximation results for more restricted option pricing problems~\cite{elbrachter2018dnn}, and there are very recent approximation results for semilinear heat equations~\cite{hutzenthaler2019proof} and Kolmogorov equations with (time-inhomogeneous) nonlinear coefficients~\cite{jentzen2018proof,reisinger2019rectified} where the dependence on the dimension is polynomial. Using a generalized version of Lemma~\ref{lem:kolm_reg}, the findings of this paper can be used to prove that the corresponding ERM problem achieves, with high probability, a desired accuracy $\varepsilon$ with the number of samples and the size of the hypothesis class scaling only
polynomially in $d$ and $\varepsilon^{-1}$.
In particular this means that the presented methods are not restricted to the case of linear Kolmogorov equations with affine drift and diffusion coefficients. Finally, note that one obtains similar results for any continuous piecewise linear activation function with a finite number of breakpoints; see the comment after Theorem~\ref{thm:nn_lip} and~\cite[Proposition 1]{yarotsky2017error}.
\subsection{Outline}
The outline is as follows. In Section~\ref{sec:mainresult} we present our main result related to the generalization of neural networks in a rather general setting. Whenever the regression functions $(f^*_{d})_{d\in\N}$ can be approximated without the curse of dimensionality by clipped ReLU networks, we show that also the number $m$ of required training samples to achieve a desired accuracy $\varepsilon$ with high probability does not suffer from the curse of dimensionality. This result is proven using tools from statistical learning theory and covering number estimates of neural network hypothesis classes. 
In Section~\ref{sec:pde} we reformulate the numerical approximation of $F_d(T,\cdot)$ on $[u,v]^d$ as a statistical learning problem and extend a result of~\cite{grohs2018approx} claiming that the end values $(F_d(T,\cdot))_{d\in\N}$ can be approximated by clipped ReLU networks without the curse of dimensionality. Therefore our results from Section~\ref{sec:mainresult} apply and give rise to quantitative polynomial bounds on the number of samples and the size of the network in Theorem~\ref{thm:kolm_nngen}. 

\section{Results in statistical learning theory}\label{sec:mainresult}
The present section develops generalization bounds for ERM problems in the spirit of \cite{anthony2009,cucker2002,cucker_zhou_2007}.

\subsection{A generalization result based on covering numbers}\label{sec:gen_learn}
Setting \ref{set:MathLearn} describes a standard statistical learning problem.
\begin{setting}[statistical learning problem]\label{set:MathLearn}
Let
$u\in\R$,
$v\in(u,\infty)$, and
$\bd\in [1,\infty)$, and let 
$(\Omega,\mathcal{G},\P)$ be a probability space.
For every $d\in \N$ let 
$$X_d:\Omega\to [u,v]^d \quad\text{(\emph{input data})}\quad \text{and} \quad Y_d:\Omega\to  [-\bd,\bd] \quad \text{(\emph{label})}$$
be random variables, 
let 
$\PX$ 
be the image measure of 
$X_d$ 
on the hypercube 
$[u,v]^d$, let
$$(X_d^{(i)},Y_d^{(i)})\colon \Omega \to [u,v]^d\times [-\bd,\bd],\quad i\in\N,\quad \text{(\emph{training data})}$$
be i.i.d.\@ random variables with $(X_d^{(1)},Y_d^{(1)}) \sim (X_d,Y_d)$, and let
$f^*_{d}\in L^2\left(\PX\right)$
satisfy\footnote{We define the Hilbert 
space $\leb^2\left(\PX\right)$ as the space of all Borel measurable functions $f\colon [u,v]^d\to\R$ with finite norm $\|f\|_{\leb^2(\PX)}=\big(\int_{[u,v]^d} f^2 \, d\PX\big)^{1/2}
< \infty $ where functions coinciding $\PX$-a.s.\@ are identified as usual.} for $\PX$-a.s.\@ $x\in[u,v]^d$ that 
\begin{equation*}
    f^*_{d}(x)=\E\left[Y_d \big| X_d=x\right] \quad \text{(\emph{regression function})}.
\end{equation*}
For every $d,m\in \N$ and Borel measurable function $f\colon[u,v]^d\to\R$ define the \emph{risk} $\cE_d(f)\in[0,\infty]$ and the \emph{empirical risk} $\widehat{\cE}_{d,m}(f)\colon \Omega\to [0,\infty)$ by
\begin{equation*}
\cE_d(f):=\E\Big[\big(f(X_d)-Y_d\big)^2\Big] \quad \text{and} \quad \widehat{\cE}_{d,m}(f):=\tfrac{1}{m}\sum_{i=1}^{m}\big(f(X_d^{(i)})-Y_d^{(i)}\big)^2.
\end{equation*}
For every $d,m\in\N$ and every compact $\cH\subseteq \cont([u,v]^d,\R)$ (\emph{hypothesis class}) let 
\begin{equation} \label{eq:def_best_approx} f_{d,\cH}\in \mathop{\mathrm{argmin}}_{f\in \cH}\cE_{d}(f) \quad \text{(\emph{best approximation})},
\end{equation}
and for every $d,m\in\N$, $\omega\in\Omega$ and every compact $\cH\subseteq \cont([u,v]^d,\R)$ let
\begin{equation} \label{eq:def_empirical_reg}
 \widehat f_{d,m,\cH}(\omega)\in \mathop{\mathrm{argmin}}_{f\in \cH}\widehat{\cE}_{d,m}(f)(\omega) \quad \text{(\emph{empirical regression function})}
\end{equation}
such that the mapping
$\Omega\ni \omega \mapsto \widehat f_{d,m,\cH}(\omega)$
is measurable.
\end{setting}

We want to emphasize that the minima in~\eqref{eq:def_best_approx} and~\eqref{eq:def_empirical_reg} will be attained due to the compactness of our hypothesis class, but they need not be unique. 
We require the mapping $\Omega\ni \omega \mapsto \widehat f_{d,m,\cH}(\omega)$ to be measurable in order to view the risk of the empirical regression function as a random variable $\Omega\ni \omega \mapsto \cE_d\big(\widehat f_{d,m,\cH}(\omega)\big)$. This ensures that the probability in the generalization error bound (Theorem~\ref{thm:bound_sample}) is well-defined.
While this technical assumption is often not explicitly stated in the literature on statistical learning theory it is actually crucial for analyzing the generalization error. In our setting (by choosing a suitable minimizer) measurability can indeed be satisfied; see Appendix~\ref{app:meas}.

The following lemma states that the regression function $f^*_{d}$ indeed minimizes the risk and the function $f_{d,\cH}$ is a best approximation 
of $f^*_{d}$ in $\cH$ with respect to the $\leb^2(\PX)$-norm. Moreover, it offers a decomposition of the error between the empirical regression function $\widehat f_{d,m,\cH}$ and the regression function $f^*_{d}$, often referred to as the bias-variance decomposition. 

\begin{lemma}[bias-variance decomposition]\label{lem:biasvariance}
Assume Setting \ref{set:MathLearn}. Let $d,m\in \N$ and let $\cH\subseteq \cont([u,v]^d,\R)$ be compact. Then it holds that
\begin{enumerate}[label=(\roman*)]
\item \label{it:bv1} $\cE_d\big(f^*_{d}\big) = \min_{f\in \leb^2(\PX)} \cE_d\big(f\big)$,
\item \label{it:bv2} $ \big\|f_{d,\cH}-f^*_{d}\big\|_{\leb^2(\PX)} = \min_{f\in\cH}
\big\|f-f^*_{d}\big\|_{\leb^2(\PX)}$, and
\item \label{it:bv3} $ \big\|\widehat f_{d,m,\cH}-f^*_{d}\big\|_{\leb^2(\PX)}^2 = \underbrace{\cE_d\big(\widehat f_{d,m,\cH}\big)-\cE_d\big(f_{d,\cH}\big)\vphantom{\big)_{\leb^2(\PX)}^2}}_{\textnormal{generalization error (variance)}} +
\underbrace{\big\|f_{d,\cH}-f^*_{d}\big\|_{\leb^2(\PX)}^2.}_{\textnormal{approximation error (bias)}}
$
\end{enumerate}
\end{lemma}
For a corresponding result see also~\cite{cucker2002}. The proof in Appendix~\ref{app:bias_var} is based on the fact that we consider the square loss, and thus the risk of $f\in \leb^2(\PX)$ can be represented as
\begin{equation*}
\cE_d(f) = \big\|f-f^*_{d}\big\|_{\leb^2(\PX)}^2 +\cE_d(f^*_{d}).
\end{equation*}
We now introduce the concept of covering numbers in order to bound the generalization error. 
\begin{setting}[covering number]\label{set:cover}
	For every $\rr\in(0,1)$, every normed vector space $(\mathcal{Z},\|{\cdot}\|)$, and every compact subset $\cH\subseteq \mathcal{Z}$ we define the \emph{$\rr$-covering number of $\cH$ w.r.t.\@ $\|{\cdot}\|$} by 
    \begin{equation*}
        \cov(\cH,\|{\cdot}\|,\rr):=\min\Big\{n\in \N:\ \text{There\ exists\ } (f_i)_{i=1}^n\subseteq \cH\ \text{ with }\  \cH\subseteq \bigcup_{i=1}^n \mathit{Ball}_{\rr}(f_i)\Big\},
    \end{equation*}
    where $\mathit{Ball}_{\rr}(f):=\{g\in \cH:\ \|f-g\| \le \rr\}$ denotes the \emph{ball of radius $r$ around $f\in\cH$}. If the norm is clear from the context, we will use the abbreviation $\cov(\cH,\rr):=\cov(\cH,\|{\cdot}\|,\rr)$.
\end{setting}
Assume that the functions in our hypothesis class $\cH$ are uniformly bounded and that balls of radius $r$ around the functions $(f_i)_{i=1}^n$ cover $\cH$. 
We can then use the (uniform) Lipschitz continuity of the (empirical) risk to bound the generalization error by 
\begin{equation*}
\begin{split}
   \cE_d\big(\widehat f_{d,m,\mathcal{H}}\big)-\cE_d\big(f_{d,\mathcal{H}}\big) &\le \cE_d\big(\widehat f_{d,m,\mathcal{H}}\big)-\widehat{\cE}_{d,m}\big(\widehat f_{d,m,\mathcal{H}}\big)+\widehat{\cE}_{d,m}\big(f_{d,\mathcal{H}}\big)-\cE_d\big(f_{d,\mathcal{H}}\big) \\
   &\le 2\rr \big[\operatorname{Lip}(\cE_d) +\operatorname{Lip}(\widehat{\cE}_{d,m})\big] + 2\max_{i=1}^n \left|\cE_d(f_i)-\widehat{\cE}_{d,m}(f_i) \right|
\end{split}
\end{equation*}
and employ Hoeffding's inequality and a union bound to obtain the following estimate.
\begin{theorem}[generalization error bound] \label{thm:bound_sample}
Assume Settings \ref{set:MathLearn} and \ref{set:cover}.
Let
$\varepsilon\in(0,1)$,
$d, m \in\N$
and let $\cH\subseteq \cont([u,v]^d,\R)$ be compact with 
$
    \sup_{f\in\cH} \|f\|_{\leb^\infty} \le D.
$
Then it holds that
\begin{equation*}
\P\left[ \cE_d\big(\widehat f_{d,m,\cH}\big)-\cE_d\big(f_{d,\cH}\big)\le \varepsilon\right] \ge 1-2\cov\Big(\cH,\frac{\varepsilon}{32\bd}\Big)\exp\left(-\frac{m\varepsilon^2}{128\bd^4}\right).
\end{equation*}
\end{theorem}
This result is adapted from~\cite[Theorem 17.1]{anthony2009},~\cite[Theorem 3.14]{cucker_zhou_2007}, and~\cite[Example 3.31]{mohri2012}, and for the sake of completeness we present a detailed proof in Appendix~\ref{app:gen}. 
\subsection{Covering numbers of neural network hypothesis classes} \label{subsec:covering}
As a natural next step we prove estimates on the covering numbers of neural network hypothesis classes in order to leverage the result of Theorem~\ref{thm:bound_sample}. Note that for different assumptions (i.e., boundedness assumptions on the activation function, different norms on the parameters, or evaluation of the neural networks on input data) similar approaches can be found in~\cite{anthony2009,BartlettFT17}.

The following setting describes suitable hypothesis classes based on neural networks.
From now on we only consider neural networks with ReLU activation function and therefore omit writing the index $\rho=\ReLU$ in our notation.
\begin{setting}[neural networks]\label{set:NN}
Assume Setting~\ref{set:MathLearn}. For 
every $k,n\in\N$,
$W\in\R^{k\times n}$,
$B\in\R^k$ let $\cA_{W,B}\in \cont(\R^n,\R^k)$ be the affine mapping which satisfies for every $x\in\R^n$ that $\cA_{W,B}(x):=Wx+B$. For every $n\in\N$, $x=(x_i)_{i=1}^n\in\R^n$ we define
\begin{equation*} 
\ReLU_*(x):=\big(\max\{x_i,0\}\big)_{i=1}^n \quad \text{(\emph{componentwise rectified linear unit})}.
\end{equation*}
For every 
$L\in\N$,
$\ba=(a_0,a_1,\dots,a_{L})\in \N^{L+1}
$
(\emph{network architecture}) we define
\begin{equation*}
\p_{\ba} := \bigtimes_{l=1}^L \big( \R^{a_{l}\times a_{l-1}} \times \R^{a_{l}} \big)  \quad \text{(\emph{set of neural network parametrizations})},
\end{equation*}
$L(\ba):=L$ (\emph{number of layers}), and
\begin{equation*}
    \sz(\ba):=\sum_{l=1}^{L} a_{l} a_{l-1}+a_{l} \quad \text{(\emph{number of parameters})}.
\end{equation*}
For every $L\in\N$, $\ba\in \N^{L+1}$, $\ptheta=((W_l,B_l))_{l=1}^{L} \in \p_{\ba}$ we define the \emph{neural network realization function} $\F(\ptheta)\in \cont(\R^{a_0},\R^{a_{L}})$ by
\begin{equation*}
    \F(\ptheta):=  \cA_{W_L,B_L} \circ \ReLU_* \circ \, \cA_{W_{L-1},B_{L-1}} \circ \ReLU_* \circ \dots \circ \ReLU_* \circ \,\cA_{W_1,B_1} 
\end{equation*}
and its restriction $\F^{u,v}(\ptheta):=\F(\ptheta)|_{[u,v]^{a_{\scriptsize 0}}}\in \cont([u,v]^{a_0},\R^{a_{L}})$ to the hypercube $[u,v]^{a_0}$. 
For every $d\in \N$ let the \emph{admissible network architectures} be given by $$\bA_d:=\bigcup_{L\in \N}\big\{(a_0,a_1,\dots , a_L)\in \N^{L+1}:\ a_0 = d,\ a_{L}=1\big\}.$$
For every $d\in\N$, $\ba\in\bA_d$, $R\in(0,\infty)$ (\emph{parameter bound}) define
the \emph{set of bounded neural network parametrizations} 
\begin{equation*}
\p_{\ba,R} := \bigtimes_{l=1}^L \big( [-R,R]^{a_{l}\times a_{l-1}} \times [-R,R]^{a_{l}} \big)=\{\ptheta\in\p_a\colon \|\ptheta\|_\infty\le R\},
\end{equation*}
the \emph{hypothesis class of neural networks} 
$$\cN_{\ba,R}^{u,v}:=\F^{u,v}(\p_{\ba,R})=\left\{\Big([u,v]^d 
\ni x\mapsto \F(\ptheta)(x)\Big):\ \ptheta \in \p_{\ba,R}
\right\},$$
and the \emph{hypothesis class of clipped neural networks} 
\begin{equation*}
\cN_{\ba,R,\bd}^{u,v}:=\left\{ \cC_{\bd}\circ  g \colon g\in  \cN_{\ba,R}^{u,v}
\right\},\end{equation*} 
where $\cC_{\bd} \in \cont(\R, \R)$ (\emph{clipping function}) satisfies for every $x\in\R$ that
\begin{equation*}
    \cC_\bd(x):=\min\{|x|,\bd\}\mathop{\mathrm{sgn}}(x).
\end{equation*}
\end{setting}
The hypothesis classes $\cN_{\ba,R,\bd}^{u,v}$ are somewhat nonstandard in the sense that
the clipping function $\cC_{\bd}$ is applied to the output of a neural network realization.
The reason for our choice of this definition is that Theorem~\ref{thm:bound_sample} requires that the set of neural networks over which the ERM problem is solved consists of uniformly bounded functions. 
In Appendix~\ref{app:nn} we show that the clipping function $\cC_{\bd}$ can be represented as a small neural network, which implies that the seemingly nonstandard classes $\cN_{\ba,R,\bd}^{u,v}$ are actually conventional neural network classes that can be trained with standard methods~\cite{hinton2012,Kingma2015,lecun2015,nielsen2015}.

The next theorem quantifies the Lipschitz continuity $\operatorname{Lip}(\F^{u,v})$ of the operator \begin{equation*} 
    \F^{u,v}\colon(\p_{\ba,R},\|{\cdot}\|_\infty) \to (\cN_{\ba,R}^{u,v},\|{\cdot}\|_{\leb^\infty})
\end{equation*} which maps bounded neural network parametrizations with fixed architecture to the corresponding realization functions (restricted to $[u,v]^d$). 
\begin{theorem}[Lipschitz continuity of $\F$] \label{thm:nn_lip}Assume Setting \ref{set:NN}. Let $d\in\N$,
$\ba\in \bA_d$, and
$R\in[1,\infty)$.
Then for every
$\ptheta,\peta \in \p_{\ba,R}$
it holds that
\begin{equation*}
\| \F^{u,v}(\ptheta) -  \F^{u,v}(\peta) \|_{\leb^\infty} \le  2\max \big\{1,|u|,|v|\big\} \cL(\ba)^2 R^{\cL(\ba)-1}\|\ba\|_\infty^{\cL(\ba)} \|\ptheta -\peta \|_\infty.
\end{equation*}
\end{theorem}
The proof is based on estimating the error amplification in each layer of the neural network and can be found in Appendix~\ref{app:cov}. Similar results can be established for any Lipschitz continuous activation function; see also~\cite{petersen2018topological} for a general nonquantitative result. Note that Theorem~\ref{thm:nn_lip} in particular implies that $\cN_{\ba,R}^{u,v}$ and $\cN_{\ba,R,\bd}^{u,v}$ are compact subsets of $\cont([u,v]^d,\R)$ and thus valid hypothesis classes, as required by Setting~\ref{set:MathLearn}.
Next, we recall a basic result on the covering number of a hypercube w.r.t.\@ the maximum norm $\|{\cdot}\|_\infty$. Note that a similar statement holds for any ball in a finite-dimensional Banach space; see~\cite[Proposition 5]{cucker2002}.
\begin{lemma}[covering numbers of balls]
\label{lem:cov_finite}
Assume Setting \ref{set:cover}.
Let 
$n\in\N$,
$R\in[1,\infty)$, and
$\rr\in(0,1)$,
and define
$
\mathit{Ball}_{R}:=\big\{ \ptheta \in \R^n: \|\ptheta\|_\infty \le R \big\}$. 
Then it holds that
\begin{equation*}
\ln \cov(\mathit{Ball}_{R},\|{\cdot}\|_\infty,\rr) \le n \ln\left\lceil \frac{R}{\rr} \right\rceil. 
\end{equation*}
\end{lemma}
\begin{proof}[Proof of Lemma~\ref{lem:cov_finite}]
The claim follows by a simple counting argument.
\end{proof}
Together with Theorem~\ref{thm:nn_lip} this allows us to bound the covering number of our hypothesis class of (clipped) neural networks. 
\begin{proposition}[covering numbers of neural network hypothesis classes]
\label{prop:cov_nn}
Assume Settings~\ref{set:cover} and~\ref{set:NN}. Let 
$d\in\N$,
$\ba\in\bA_d$,
$\rr\in(0,1)$, and
$R\in[1,\infty)$.
Then it holds that
\begin{equation*}
\begin{split}
\ln \cov\big(\cN_{\ba,R,\bd}^{u,v},r\big) &\le \ln \cov\big(\cN_{\ba,R}^{u,v},r\big) \\
&\le \sz(\ba)\Big[ \ln\Big( \frac{4\cL(\ba)^2 \max\big\{1,|u|,|v|\big\} }{\rr}\Big) + \cL(\ba) \ln\big( R \|\ba\|_\infty \big)\Big] .
\end{split}
\end{equation*}
\end{proposition}
The proof in Appendix~\ref{app:cov_nn_hyp} is based on the behavior of covering numbers under the action of a Lipschitz function, i.e., 
\begin{equation*}
    \cov\big(\F^{u,v}(\p_{\ba,R}),r\big) \le \cov\Big(\p_{\ba,R},\frac{\rr}{\operatorname{Lip}(\F^{u,v})}\Big),
\end{equation*}
and uses the facts that the clipping function is nonexpansive, i.e., $\operatorname{Lip}(\cC_D)=1$, and that 
$\p_{\ba,R} \simeq \big\{ \ptheta \in \R^{\sz(\ba)}: \|\ptheta\|_\infty \le R \big\}$.

\subsection{Analysis of the generalization error} \label{sec:gen_nn}
Combining Theorem~\ref{thm:bound_sample} and Proposition~\ref{prop:cov_nn}, the following theorem describes our main result related to the generalization capabilities of hypothesis classes consisting of clipped ReLU networks. 
\begin{theorem}[neural network generalization error bound]\label{thm:main}
Assume Setting~\ref{set:NN}. Let 
$h\in \cont((0,\infty)^5,\R)$ satisfy 
for every $x=(x_i)_{i=1}^5\in (0,\infty)^5$ that
\begin{equation*}
h(x)=128\bd^4x_1^2\Big[\ln(2)+x_2+x_3x_4x_5+x_4\ln\big(128\bd\max\{1,|u|,|v|\}x_1x_5^2\big)\Big],
\end{equation*}
let $d,m\in \N$, $\varepsilon,\varrho\in (0,1)$, $\mathbf{a}\in \bA_d$, 		$R\in [1,\infty)$ with
\begin{equation*}
m\geq h\left(\varepsilon^{-1},\ln(\varrho^{-1}),\ln(R\|\mathbf{a}\|_\infty),\sz(\ba),\cL(\ba)\right),
\end{equation*}
and define $\cH:=\mathcal{N}_{\mathbf{a},R,\bd}^{u,v}$.
Then it holds that
\begin{equation*} 
\P\left[ \cE_d\big(\widehat f_{d,m,\cH}\big)- \cE_d\big(f_{d,\cH}\big) \le \varepsilon\right] \geq 1-\varrho.
\end{equation*}
\end{theorem}
\begin{proof}[Proof of Theorem~\ref{thm:main}]
Proposition~\ref{prop:cov_nn} implies that
\begin{equation*}
\begin{split}
m &\geq h\left(\varepsilon^{-1},\ln(\varrho^{-1}),\ln(R\|\mathbf{a}\|_\infty),\sz(\ba),\cL(\ba)\right) \\
&= 128\bd^4 \varepsilon^{-2} \Big[\ln(2\varrho^{-1})+ \sz(\ba)\Big( \ln\big( 128\bd \max\left\{1,|u|,|v|\right\}\varepsilon^{-1} \cL(\ba)^2 \big) + \cL(\ba) \ln\big( R \|\ba\|_\infty\big) \Big) \Big] \\
& \ge 128\bd^4 \varepsilon^{-2} \left[\ln(2\varrho^{-1})+ \ln \cov\Big(\cH,\frac{\varepsilon}{32\bd}\Big) \right].
\end{split}
\end{equation*}
Now Theorem \ref{thm:bound_sample} and a simple calculation ensures that 
\begin{equation*}
\begin{split}
\P\left[ \cE_d\big(\widehat f_{d,m,\cH}\big)-\cE_d\big(f_{d,\cH}\big)\le \varepsilon\right] &\ge 1-2\cov\left(\cH,\frac{\varepsilon}{32\bd}\right)\exp\left(-\frac{m\varepsilon^2}{128\bd^4}\right) 
\ge 1-\varrho,
\end{split}
\end{equation*}
and this proves the theorem.
\end{proof}
Next we show how Theorem \ref{thm:main} can be used to leverage bounds on the approximation error in order to obtain quantitative bounds on the generalization error.
\begin{corollary}[approximation implies generalization]\label{cor:appgen}
Assume Setting~\ref{set:NN}. Let $d\in \N$, $\varepsilon\in(0,1)$, 
$\mathbf{a}\in \bA_d$, $R\in [1,\infty)$,
and 
$g\in\cH:=\mathcal{N}_{\mathbf{a},R,\bd}^{u,v}$
with
$$
\big\|g-f^*_{d} \big\|_{\leb^2(\PX)}^2\le \varepsilon/2,
$$
let 
$h\in \cont((0,\infty)^5,\R)$ 
be given as in Theorem~\ref{thm:main}, and
let
$m\in \N$, $\varrho \in (0,1)$ with
\begin{equation*}
m \geq h\left(2\varepsilon^{-1},\ln(\varrho^{-1}),\ln(R\|\mathbf{a}\|_\infty),\sz(\ba),\cL(\ba)\right).
\end{equation*}
Then it holds that
\begin{equation*} 
\P\left[\big\|\widehat f_{d,m,\cH}-f^*_{d}\big\|_{\leb^2(\PX)}^2 \le \varepsilon\right] \geq 1-\varrho.
\end{equation*}
\end{corollary}
\begin{proof}[Proof of Corollary~\ref{cor:appgen}]
Lemma~\ref{lem:biasvariance}
ensures that
\begin{equation*}
 \big\|f_{d,\cH}-f^*_{d}\big\|_{\leb^2(\PX)}^2 \le \big\|g - f^*_{d} 
 \big\|_{\leb^2(\PX)}^2\le \varepsilon/2
\end{equation*}
and hence
$
\big\| \widehat f_{d,m,\cH}-f^*_{d}\big\|_{\leb^2(\PX)}^2\le  \cE_d(\widehat f_{d,m,\cH}) 
-\cE_d(f_{d,\cH})+\varepsilon/2.
$
Theorem \ref{thm:main} (with $\varepsilon \leftarrow \varepsilon/2$)
now directly implies the desired claim.
\end{proof}
The previous result in particular implies that whenever the family $(f^*_{d})_{d\in \N}$ from the statistical learning problem of Setting \ref{set:MathLearn} can be approximated by neural networks without the curse of dimensionality, then the number $m$ of required training samples to achieve a desired accuracy with high probability does not suffer from the curse of dimensionality either.
A compact version of this statement is given in the next result.
\begin{corollary}[approximation without curse implies generalization without curse]
\label{cor:appgen2} Assume Setting~\ref{set:NN}. 
Assume that there exists a polynomial $q:\R^2\to \R$ such that for every $d\in \N$, $\varepsilon \in (0,1)$
there exist $\mathbf{a}_{d,\varepsilon}\in \bA_d$, $R_{d,\varepsilon}\in [1,\infty)$, and $g_{d,\varepsilon}\in \cH_{d,\varepsilon}:= \mathcal{N}_{\mathbf{a}_{d,\varepsilon},R_{d,\varepsilon},\bd}^{u,v}$ with 
\begin{equation*}
\max\left\{\ln(R_{d,\varepsilon}),\sz(\mathbf{a}_{d,\varepsilon})\right\} \le q(d,\varepsilon^{-1}) \quad \text{and} \quad  \big\| g_{d,\varepsilon}-f^*_{d} \big\|_{\leb^2(\PX)}^2\le \varepsilon/2.
\end{equation*}
Then there exists a polynomial $s:\R^2\to \R$ such that for every $d,m\in\N$, $\varepsilon,\varrho\in (0,1)$ with
\begin{equation*}
m\geq s(d,\varepsilon^{-1})(1+\ln(\varrho^{-1}))
\end{equation*}
it holds that
\begin{equation*}
\P\left[\big\|\widehat f_{d,m,\cH_{d,\varepsilon}}-f^*_{d}\big\|_{\leb^2(\PX)}^2 \le \varepsilon\right] \geq 1-\varrho.
\end{equation*}
\end{corollary}
\begin{proof}[Proof of Corollary~\ref{cor:appgen2}]
Observe that for every $d\in\N$, $\varepsilon\in(0,1)$ it holds that
\begin{equation*}
\max\left\{\ln(\|\mathbf{a}_{d,\varepsilon}\|_\infty),\cL(\mathbf{a}_{d,\varepsilon})\right\}\le \sz(\mathbf{a}_{d,\varepsilon}) \le q(d,\varepsilon^{-1})
\end{equation*}
and that the function $h\in \cont((0,\infty)^5,\R)$ from Theorem~\ref{thm:main} satisfies for every $x\in(0,\infty)^5$ that
\begin{equation*}
h(x)\le 128\bd^4x_1^2\Big[1+x_4\Big(x_1+x_3x_5+2x_5+\ln\big(128\bd\max\{1,|u|,|v|\}\big)-3\Big)\Big](1+x_2).
\end{equation*}
Thus, Corollary~\ref{cor:appgen2} is a direct consequence of Corollary \ref{cor:appgen}.
\end{proof}

\section{Applications for the numerical approximation of high-dimensional PDEs}\label{sec:pde}
In the present section we apply the general results of Section \ref{sec:mainresult} to the numerical solution of high-dimensional Kolmogorov equations. 
\subsection{Kolmogorov equation as learning problem}\label{sec:KolmoLearn}
The following setting describes suitable Kolmogorov equations and the data for the corresponding statistical learning problems.
\begin{setting}[Kolmogorov equations]\label{set:PDE}Assume Setting~\ref{set:NN}. 
Let $K\in (0,\infty)$, for every $d\in \N$ let $\mu_d\in \cont(\R^d,\R^d)$ (\emph{drift coefficient}) and $\sigma_d\in \cont(\R^d,\R^{d\times d})$ (\emph{diffusion coefficient}) be affine functions satisfying for every $x\in \R^d$ that
$$
\|\sigma_d(x)\|_\eu+\|\mu_d(x)\|_{\eu}\le K(1 + \|x\|_{\eu}),
$$ 
and let $\varphi_d\in \cont(\R^d,[-\bd,\bd])$ (\emph{initial value}). Assume that $(\varphi_d)_{d\in\N}$ can be approximated by neural networks in the following sense: Let $\zeta \in [1,\infty)$ and 
$\beta,\gamma,\kappa,\lambda, \nu \in [0,\infty)$, and
let 
$$\mathbf{b}_{d,\varepsilon}\in \bA_d, \quad \peta_{d,\varepsilon} \in \p_{\mathbf{b}_{d,\varepsilon}}, \quad d\in\N,
\ \varepsilon\in(0,1), \quad \text{(\emph{neural network approximation of $\varphi_d$})}$$
such that for every 
$d\in \N$, $\varepsilon\in (0,1)$, $x\in\R^d$
it holds that\footnote{Due to the boundedness assumption on $\varphi_d$ one can obtain the desired estimates in item~\ref{it:set1} and~\ref{it:set2} by adapting known neural network approximation results; see~\cite{bernertowards} and Appendix~\ref{app:nn}.}
\begin{enumerate}[label=(\roman*)]
\item \label{it:set1}
$
|\varphi_d(x)-\F(\peta_{d,\varepsilon})(x)|\le \varepsilon (1+\|x\|_{\eu}^{\nu}),
$
\item \label{it:set2}
$|\F(\peta_{d,\varepsilon})(x)|\le D$,
\item
$\|\peta_{d,\varepsilon}\|_\infty \le \zeta d^{\beta}\varepsilon^{-\kappa},
$ and
\item
$
\sz(\mathbf{b}_{d,\varepsilon})\le \zeta d^{\gamma}\varepsilon^{-\lambda}.
$
\end{enumerate}
Let $T\in (0,\infty)$, and for every $d\in \N$ let $F_d\in \cont([0,T]\times \R^d,\R)$ be the unique\footnote{For a proof see~\cite[Proposition 3.4(i)]{grohs2018approx}.
} function satisfying the following:
\begin{enumerate}[label=(\roman*)]
\item $F_d(0,x)=\varphi(x)$ for every $x\in\R^d$;
\item $F_d$ is at most polynomially growing, i.e., there exists $\vartheta\in(0,\infty)$ such that for every $x\in\R^d$ it holds that
$ \max_{t\in[0,T]} F_d(t,x) \le \vartheta \left(1+\|x\|_2^\vartheta\right)$; and
\item $F_d$ is a viscosity \emph{solution of the $d$-dimensional Kolmogorov equation}
\begin{equation*}
\tfrac{\partial F_d}{\partial t}(t,x)=\tfrac12\mathrm{Trace}\big(\sigma_d(x)[\sigma_d(x)]^* (\mathrm{Hess}_x F_d)(t,x)\big) 
+  \mu_d(x) \cdot (\nabla_x F_d)(t,x) 
\end{equation*}
for every $(t,x)\in (0,T)\times \R^d$.
\end{enumerate}
Let the probability space
$(\Omega,\G,\P)$
be equipped with a filtration 
$(\G_t)_{t\in [0,T]}$ which fulfills the usual conditions. For every $d\in\N$ 
let
$(B^d_t)_{t\in[0,T]}\colon [0,T]\times\Omega\to\R^d$
be a
$d$-dimensional
$(\G_t)$-Brownian
motion, and for every $\G_0$-measurable random variable $\chi\colon\Omega\to\R^d$ denote by
$$(S_t^\chi)_{t\in[0,T]}\colon [0,T]\times\Omega\to\R^d \quad \text{(\emph{SDE solution process with initial value $\chi$})}$$ the unique $(\G_t)$-adapted stochastic process\footnote{The solution process is unique up to indistinguishability; see, for instance,~\cite[Theorem 6.2.2]{arnold1974}.} with continuous sample paths satisfying the SDE
\begin{equation*} 
dS_t^{\chi}=\sigma_d(S_t^{\chi})dB_t^d + \mu_d(S_t^{\chi})dt\quad \text{and}\quad S^{{\chi}}_0={\chi}
\end{equation*}
$\P$-a.s.\@ for every $t\in [0,T]$.
For every $d\in\N$ let the \emph{input data}
$X_d\colon\Omega\to [u,v]^d$ 
be $\G_0$-measurable and uniformly distributed on 
$[u,v]^d$ and define the \emph{label} by 
$Y_d := \varphi_d(S^{X_d}_T)$.
\end{setting}
The next result shows that computing the end value $[u,v]^d\ni x\mapsto F_d(T,x)$ of the solution to the Kolmogorov equation  can be restated as a learning problem.
\begin{lemma}[Kolmogorov equation as learning problem]
\label{lem:kolm_reg}
Assume Setting~\ref{set:PDE} and let $d\in\N$. Then for a.e.\@ $x\in [u,v]^d$ it holds that
$$
F_d(T,x) = f^*_{d}(x).
$$
\end{lemma}
The result is based on work from~\cite{beckbecker2018} and the following formal calculation:
\begin{equation*}
F_d(T,x)=\E\big[\varphi_d\left(S_T^x\right)\big]=\E\big[\varphi_d(S^{X_d}_T)\big|X_d=x\big]=\E\left[Y_d\big|X_d=x\right]=f^*_{d}(x)
\end{equation*}
for a.e.\@ $x\in [u,v]^d$; see Appendix~\ref{app:learn} for a rigorous proof. 
\subsection{Neural network generalization results for solutions of Kolmogorov equations}\label{sec:pdegen}
We first show that the end value $[u,v]^d\ni x\mapsto F_d(T,x)$ of the solution to the Kolmogorov equation can be approximated by hypothesis classes consisting of clipped ReLU networks.
\begin{theorem}[neural network regularity result for Kolmogorov equations]\label{thm:mainapprox}
Assume Setting~\ref{set:PDE}.
Then there exist $C,c\in (0,\infty)$ such that the following holds: 
For every $d\in \N$, $\varepsilon\in (0,1)$ there exist $\ba\in \bA_d$, $R\in[1,\infty)$, and 
$g\in \mathcal{N}_{\mathbf{a},R,\bd}^{u,v}$ such that it holds that
\begin{enumerate}[label=(\roman*)]
\item
$
\tfrac{1}{(v-u)^d}\big\|g-F_d(T,{\cdot})\big\|_{\leb^2([u,v]^d)}^2\le 
\varepsilon,
$
\item 
$
\sz(\ba)\le Cd^{\nu\lambda/2+\gamma }\varepsilon^{-\lambda/2-2},
$
\item \label{it:mainapprox2}
$R\le Cd^{(\nu\kappa+3)/2+\beta}\varepsilon^{-\kappa/2-1},$
\item
$\cL(\ba) = \cL(\mathbf{b}_{d,cd^{-\nu/2} \varepsilon^{1/2}})$, and
\item
$ \| \ba \|_\infty \le   C\varepsilon^{-1}\|\mathbf{b}_{d,cd^{-\nu/2} \varepsilon^{1/2}}\|_\infty.$
\end{enumerate}
\end{theorem}
Except for property~\ref{it:mainapprox2} a similar result was shown in \cite[Corollary 3.13]{grohs2018approx}. We present the proof in Appendix~\ref{app:approx} and briefly sketch the idea in the following. First we observe that in our case of affine coefficients 
$\sigma_d$ and $\mu_d$ 
there exist random variables
$\mathfrak{M}$
and 
$\mathfrak{N}$
such that for all 
$x\in \R^d$ 
it holds 
$\P$-a.s.\@ 
that
$
S^{x}_T=\mathfrak{M}x+\mathfrak{N}
$;
see Lemma~\ref{lem:affine}.
Let $((\mathfrak{M}^{(j)},\mathfrak{N}^{(j)}))_{j\in\N}$
be i.i.d.\@ samples with 
$(\mathfrak{M}^{(1)},\mathfrak{N}^{(1)})\sim (\mathfrak{M},\mathfrak{N})$.
Then for fixed $x\in\R^d$ the mean squared error 
\begin{equation*}
\E\Big[\Big(F_d(T,x)-\tfrac{1}{n}\sum_{j=1}^{n}\F(\peta_{d,\delta})\big(\mathfrak{M}^{(j)}x+\mathfrak{N}^{(j)}\big)\Big)^2\Big]
\end{equation*}
can be decomposed into the sum of the squared bias and the variance, i.e.,
\begin{equation*}
    \underbrace{\vphantom{\sum_{j=1}^{n}}\E\Big[\varphi_d(S^{x}_T)-\F(\peta_{d,\delta})(S^{x}_T)\Big]^2}_{\cO(\delta^2)} +
    \underbrace{\mathbbm{V}\Big[\tfrac{1}{n}\sum_{j=1}^{n}\F(\peta_{d,\delta})\big(\mathfrak{M}^{(j)}x+\mathfrak{N}^{(j)}\big) \Big]}_{\cO(n^{-1})},
\end{equation*}
where we used the Feynman--Kac formula, our assumptions, and properties of Monte Carlo sampling. 
With more effort one can prove analogous estimates in the 
$\leb^2(\PX)$-norm, and our setting implies that $\PX=\tfrac{1}{(v-u)^d} \lambda_{[u,v]^d}$, where $\lambda_{[u,v]^d}$ denotes the Lebesgue measure on $[u,v]^d$. This suggests that, given $\varepsilon\in(0,1)$, for sufficient large $n\in \N$ and small $\delta\in(0,1)$ there exists an outcome 
$\omega\in\Omega$
such that with
$ M^{(j)}:=\mathfrak{M}^{(j)}(\omega)$ and 
$N^{(j)}:=\mathfrak{N}^{(j)}(\omega)$
it holds that
\begin{equation*}
\tfrac{1}{(v-u)^d} \int_{[u,v]^d}\Big( F_d(T,x)-\tfrac{1}{n}\sum_{j=1}^{n}\F(\peta_{d,\delta})\big(M^{(j)}x+N^{(j)}\big)\Big)^2 \, dx \le \varepsilon.
\end{equation*}
In Lemma~\ref{lem:affcomb} we specify a network architecture $\ba\in\bA_d$ and a parametrization 
$\ptheta\in\p_\ba$ 
such that for every $x\in\R^d$ it holds that 
\begin{equation*}
\F(\ptheta)(x) = \tfrac{1}{n}\sum_{j=1}^{n}\F(\peta_{d,\delta})\big(M^{(j)}x+N^{(j)}\big)
\end{equation*}
and we bound the parameter magnitudes of $\ptheta$ with the help of Lemma~\ref{lem:affine}. 

Observe that the approximation result in Theorem~\ref{thm:mainapprox} does not underlie the curse of dimensionality, and by Corollary~\ref{cor:appgen} we can establish a generalization result that is free of the curse of dimensionality.
\begin{theorem}[neural network generalization result for Kolmogorov equations] \label{thm:kolm_nngen}
Assume Setting~\ref{set:PDE} and let $h\in \cont((0,\infty)^5,\R)$ 
be given as in Theorem~\ref{thm:main}.
Then there
exist $C,c\in (0,\infty)$ such that the following holds: 
For every 
$d, m\in \N$, $\varepsilon,\varrho\in (0,1)$ with
\begin{equation*}
m\geq h\left(2\varepsilon^{-1},\ln(\varrho^{-1}),\ln\left( R \| \ba \|_\infty\right), \sz(\ba),\cL(\ba))\right)
\end{equation*}
there exist $\ba\in \bA_d$ and
$R\in[1,\infty)$ such that it holds that 
\begin{enumerate}[label=(\roman*)]
\item $ 
\P\left[\tfrac{1}{(v-u)^d}\big\|\widehat f_{d,m,\cH}-F_d(T,\cdot)\big\|_{\leb^2([u,v]^d)}^2 \le \varepsilon\right] \geq 1-\varrho
$,
\item 
$
\sz(\ba)\le Cd^{\nu\lambda/2+\gamma }\varepsilon^{-\lambda/2-2},
$
\item 
$R \le  Cd^{(\nu\kappa+3)/2+\beta}\varepsilon^{-\kappa/2-1},$
\item 
$\cL(\ba) = \cL(\mathbf{b}_{d,cd^{-\nu/2} \varepsilon^{1/2}})$, and
\item 
$\| \ba \|_\infty \le  C\varepsilon^{-1}\|\mathbf{b}_{d,cd^{-\nu/2} \varepsilon^{1/2}}\|_\infty$,
\end{enumerate}
where $\cH=\mathcal{N}_{\ba,R,\bd}^{u,v}.$ 
\end{theorem}
\begin{proof}[Proof of Theorem~\ref{thm:kolm_nngen}]
This is a direct consequence of Theorem \ref{thm:mainapprox} (with $\varepsilon \leftarrow \varepsilon/2$) and Corollary~\ref{cor:appgen}.
\end{proof}
We can also reformulate this in a more compact form.
\begin{corollary}[ERM for Kolmogorov equations without curse]\label{cor:kolm_main}
Assume Setting~\ref{set:PDE}.
Then there
exists a polynomial $p\colon\R^2\to\R$ such that the following holds:
For every $d,m\in \N$, $\varepsilon,\varrho\in (0,1)$ with
\begin{equation*}
    m\geq p(d,\varepsilon^{-1})(1+\ln(\varrho^{-1}))
\end{equation*}
there exist $\ba\in \bA_d$ and $R\in[1,\infty)$ such that it holds that 
\begin{enumerate}[label=(\roman*)]
\item $\P\left[\tfrac{1}{(v-u)^d}\big\|\widehat f_{d,m,\cH}-F_d(T,\cdot)\big\|_{\leb^2([u,v]^d)}^2 \le \varepsilon\right] \geq 1-\varrho$ and
\item $\max\{R,\sz(\ba)\} \le p(d,\varepsilon^{-1})$,
\end{enumerate}
where $ 
\cH=\mathcal{N}_{\ba,R,\bd}^{u,v}
$.
\end{corollary}
\begin{proof}[Proof of Corollary~\ref{cor:kolm_main}]
This follows directly from Theorem \ref{thm:mainapprox} and Corollary~\ref{cor:appgen2}.
\end{proof}
\subsection{Pricing of high-dimensional options} \label{sec:european}
The proof of Theorem~\ref{thm:european} from the introductory section dealing with the pricing of high-dimensional European put options is now an easy consequence of the above theory. 
\begin{proof}[Proof of Theorem \ref{thm:european}]
We first show that the approximation of 
$(\varphi_d)_{d\in\N}$
by clipped neural networks according to Setting~\ref{set:PDE} is possible. Note that for every
$z\in\R$
it holds that
$ 
\min\{z,\bd\}=D-\ReLU_*(D-z). 
$ 
This implies that for every
$d\in\N$, $x\in\R^d$
it holds that
\begin{equation*}
\varphi_d(x)=\min\big\{\max\left\{\bd-c_d \cdot x,0\right\},\bd\big\} =\F (\peta_d)(x),
\end{equation*}
where 
\begin{equation*}
   \peta_d := \left(\left(-[c_d]^*,D\right), \left(-1,D\right),\left(-1,D\right)\right)\in\p_{(d,1,1,1),D}. 
\end{equation*}
Accordingly, Setting~\ref{set:PDE} is satisfied with 
\begin{equation*}
\zeta=\max\{\bd,6\}, \quad
\gamma=1, \quad
\beta=\kappa=\lambda=\nu=0, \quad
\mathbf{b}_{d,\varepsilon}=(d,1,1,1), \quad
\peta_{d,\varepsilon}=\peta_{d}.
\end{equation*}
Now Theorem~\ref{thm:kolm_nngen} and a straightforward calculation prove the claim.
\end{proof}

\newpage

\appendix
\section{Proofs}
This appendix contains various proofs and additional material omitted from the main text.

\subsection{Measurability of the empirical target function} \label{app:meas}
The following lemma shows that the empirical regression function can be chosen measurable as required in Setting~\ref{set:MathLearn}. This implies that the risk of the empirical regression function $\Omega \ni \omega\mapsto \cE_d\big(\widehat f_{d,m,\cH}(\omega)\big)$ is measurable which is necessary for bounding the generalization error in Theorem~\ref{thm:bound_sample}.
\begin{lemma}[measurability of the empirical regression function]
\label{lem:emp_meas}
Let
$u\in\R$,
$v\in(u,\infty)$,
and $d,m\in\N$, 
let 
$(\Omega,\mathcal{G},\P)$ be a probability space, let
$$(X_d^{(i)},Y_d^{(i)})\colon \Omega \to [u,v]^d\times \R,\quad i\in\N,$$ 
be random variables,
and let 
$\cH\subseteq \cont ([u,b]^d,\R)$ 
be compact. For every 
$\omega\in\Omega$ 
one can choose 
\begin{equation*}
\widehat f_{d,m,\cH}(\omega) \in \mathop{\mathrm{argmin}}_{f\in \mathcal{H}} \tfrac{1}{m}\sum_{i=1}^{m}\big(f(X_d^{(i)}(\omega))-Y_d^{(i)}(\omega)\big)^2 
\end{equation*}
in a way, such that it holds that\footnote{We denote by $\B(\mathcal{Z})$ the Borel $\sigma$-algebra of a topological space $\mathcal{Z}$.}
\begin{enumerate}[label=(\roman*)]
\item \label{it:target_1} 
$
\Omega\ni \omega \mapsto \widehat f_{d,m,\cH}(\omega) 
$ 
is 
$\G$/$\B(\cH)$-measurable and
\item
$
\Omega \ni \omega\mapsto \cE_d\big(\widehat f_{d,m,\cH}(\omega)\big)
$ 
is 
$\G$/$\B(\R)$-measurable.
\end{enumerate}
\end{lemma}
\begin{proof}[Proof of Lemma~\ref{lem:emp_meas}]
First observe that 
$\cH$
is a separable metric space induced by the uniform norm $\|{\cdot}\|_{\leb^\infty}$ and that for every
$f\in\cH$
the mapping
\begin{equation*}
\Omega\ni\omega \mapsto \widehat{\cE}_{d,m}(f)(\omega) 
\end{equation*}
is 
$\G$/$\B(\R)$-measurable. 
By the reverse triangle inequality we obtain for every 
$f,g\in\cH$ that
\begin{equation*}
\begin{split}
\big| \widehat{\cE}_{d,m}(f)^{1/2} - \widehat{\cE}_{d,m}(g)^{1/2} \big| 
\le \tfrac{1}{\sqrt{m}} \big\| (f(X_d^{(i)})-g(X_d^{(i)}))_{i=1}^m\big\|_\eu 
\le \| f-g \|_{\leb^\infty}.
\end{split}
\end{equation*}
This shows that for every 
$\omega\in\Omega$ 
the function 
$ 
\cH \ni f \mapsto \widehat{\cE}_{d,m}(f)(\omega) 
$ 
is continuous and the Measurable Maximum Theorem in~\cite[Theorem 18.19]{aliprantis2007} ensures that the set-valued function of minimizers of 
\begin{equation*} \label{eq:proof_min}
\min_{f\in\cH}\, \widehat{\cE}_{d,m}(f)
\end{equation*}
admits a measurable selector. That is to say, there exists a
$\G$/$\B(\cH)$-measurable 
mapping 
$\widehat f_{d,m,\mathcal{H}}\colon\Omega\to\cH$
such that for every
$\omega\in\Omega$ 
it holds that
$$\widehat f_{d,m,\mathcal{H}}(\omega)\in \mathop{\mathrm{argmin}}_{f\in \mathcal{H}}\widehat{\cE}_{d,m}(f)(\omega). $$
This establishes item~\ref{it:target_1}. For the proof of the second 
item observe that the risk $\cE_d\colon\cH\to \R$
is continuous and thus
$\B(\cH)$/$\B(\R)$-measurable.
Indeed, an analogous computation as for the empirical risk above shows that for $f,g\in\cH$ it holds that
\begin{equation*}
\begin{split}
\big| \cE_d(f)^{1/2} - \cE_d(g)^{1/2} \big| 
\le  \left\| f(X_d)-g(X_d)\right\|_{\leb^2(\P)} 
\le \| f-g \|_{\leb^\infty}.
\end{split}
\end{equation*}
This yields the claim as compositions of measurable functions are again measurable.
\end{proof}

\subsection{Bias-variance decomposition} \label{app:bias_var}
\begin{proof}[Proof of Lemma~\ref{lem:biasvariance}]
For every $f\in \leb^2(\PX)$ 
it holds that
\begin{equation}
\begin{split} \label{eq:ls}
\cE_d(f)&=\E\left[\big(f(X_d)-f^*_{d}(X_d)+f^*_{d}(X_d)-Y_d\big)^2\right] \\
&=\E\left[\big(f(X_d)-f^*_{d} (X_d) \big)^2\right]
+
\E\left[\big(f^*_{d}(X_d) - Y_d\big)^2\right] \\
&\quad+ 2\E\left[\big(f(X_d)-f^*_{d}(X_d)\big)\big(f^*_{d}(X_d) -Y_d\big)\right]      
\end{split}
\end{equation}
Observe that, due to the fact that it holds $\P$-a.s.\@ that $f^*_{d}(X_d)= \E\left[Y_d\big|X_d\right]$,
the tower property of the conditional expectation establishes for every $f\in \leb^2(\PX)$ that 
\begin{equation*}
\begin{split}
\E\left[\big(f(X_d)-f^*_{d}(X_d)\big)\big(f^*_{d}(X_d) -Y_d\big)\right]&=\E\left[\E\left[\big(f(X_d)-f^*_{d}(X_d))\big)\big(f^*_{d}(X_d) -Y_d\big)\Big|X_d\right]\right]\\
&=\E\left[\big(f(X_d)-f^*_{d}(X_d)\big)\big(f^*_{d}(X_d) -\E\left[Y_d\big|X_d\right]\big)\right]=0
\end{split}
\end{equation*}
which, together with \eqref{eq:ls}, implies that
\begin{equation}\label{eq:ls2}
\big\|f-f^*_{d}\big\|_{\leb^2(\PX)}^2=\E\left[\big(f(X_d)-f^*_{d}(X_d)\big)^2\right]
=\cE_d(f)-\cE_d(f^*_{d}).\end{equation}
This proves items~\ref{it:bv1} and~\ref{it:bv2} and shows that it holds that
\begin{equation*}
\big\|\widehat f_{d,m,\cH}-f^*_{d}\big\|_{\leb^2(\PX)}^2
= \cE_d(\widehat f_{d,m,\cH})-\cE_d(f_{d,\cH})+\cE_d(f_{d,\cH})-\cE_d(f^*_{d}).
\end{equation*}
Finally, applying \eqref{eq:ls2} (with $f\leftarrow f_{d,\cH}$)
proves the lemma.
\end{proof}

\subsection{Bound on the generalization error}
\label{app:gen}
\begin{proof}[Proof of Theorem~\ref{thm:bound_sample}]
First note that by assumption for every 
$f\in\cH$ 
it holds that
\begin{equation*}
|f(X_d)-Y_d| \le \|f\|_{\leb^\infty} + | Y_d| \le 2 \bd
\end{equation*}
and analogously for the samples $((X_d^{(i)},Y_d^{(i)}))_{i=1}^m$.
The elementary identity 
\begin{equation*}
(y_1-z)^2-(y_2-z)^2 = (y_1-y_2) (y_1+y_2-2z)
\end{equation*}
for real numbers $y_1,y_2,z\in\R$ and Jensen's inequality 
imply for every 
$f,g\in\cH$ that 
\begin{equation*}
\big|\cE_d(f)-\cE_d(g)\big| \le \E\left[\big|\big(f(X_d)-g(X_d)\big)\big(f(X_d)+g(X_d)-2Y_d\big)\big|\right] \le 4\bd\|f-g \|_{\leb^\infty}
\end{equation*}
and
\begin{equation*}
\begin{split}
\big|\widehat{\cE}_{d,m}(f)-\widehat{\cE}_{d,m}(g)\big| &\le \tfrac1m \sum_{i=1}^m\big|\big(f(X^{(i)}_d)-g(X^{(i)}_d)\big)\big(f(X^{(i)}_d)+g(X^{(i)}_d)-2Y^{(i)}_d\big)\big| \\
&\le 4\bd\|f-g\|_{\leb^\infty}.
\end{split}
\end{equation*}
Now define
$N:=\operatorname{Cov}\big(\mathcal{H},\|{\cdot}\|_{\leb^\infty},\frac{\varepsilon}{32\bd}\big)$ 
and choose
$
f_1,f_2,\dots,f_N \in \mathcal{H}
$ 
such that the balls
\begin{equation*}
\mathit{Ball}_i:=\left\{f\in\cH\colon \|f-f_i\|_{\leb^\infty} \le \frac{\varepsilon}{32\bd}\right\},\quad
i\in\{1,2,\dots,N\},
\end{equation*}
cover 
$\mathcal{H}$. 
This establishes that for every
$i\in\{1,2,\dots,N\}$, $f\in \mathit{Ball}_i$
it holds that
\begin{equation} \label{eq:unibound}
\begin{split}
\big|\cE_{d}(f)-\widehat{\cE}_{d,m}(f)\big| &\le \big|\cE_{d}(f)-\cE_{d}(f_i)\big|+\big|\cE_{d}(f_i)-\widehat{\cE}_{d,m}(f_i)\big|+\big|\widehat{\cE}_{d,m}(f_i)-\widehat{\cE}_{d,m}(f)\big|\\
&\le 8\bd\|f-f_i\|_{\leb^\infty} + \big|\cE_{d}(f_i)-\widehat{\cE}_{d,m}(f_i)\big|\\
&\le \varepsilon/4+ \big|\cE_{d}(f_i)-\widehat{\cE}_{d,m}(f_i)\big|.
\end{split}
\end{equation}
Our assumptions yield that for every 
$\omega \in \Omega$
it holds that
\begin{equation} \label{eq:est_emp}
\begin{split}
&\cE_d\big(\widehat f_{d,m,\mathcal{H}}(\omega)\big)-\cE_d\big(f_{d,\mathcal{H}}\big) \\
&\le \cE_d\big(\widehat f_{d,m,\mathcal{H}}(\omega)\big)-\widehat{\cE}_{d,m}\big(\widehat f_{d,m,\mathcal{H}}(\omega)\big)(\omega)+\widehat{\cE}_{d,m}\big(f_{d,\mathcal{H}}\big)(\omega)-\cE_d\big(f_{d,\mathcal{H}}\big) \\
&\le 2 \sup_{f\in\cH} \big|\cE_d(f)-\widehat{\cE}_{d,m}(f)(\omega) \big|.
\end{split}
\end{equation}
In summary~\eqref{eq:unibound} and~\eqref{eq:est_emp} imply that 
\begin{equation}
\begin{split}\label{eq:subsets}
&\Big\{\omega\in\Omega\colon \cE_d\big(\widehat f_{d,m,\mathcal{H}}(\omega)\big)-\cE_d\big(f_{d,\mathcal{H}}\big) \ge \varepsilon \Big\} \\ &\subseteq 
\bigcup_{i=1}^N \Big\{\omega\in\Omega\colon \sup_{f\in \mathit{Ball}_i} \big|\cE_d(f)-\widehat{\cE}_{d,m}(f)(\omega) \big| \ge \varepsilon/2 \Big\} \\
&\subseteq
\bigcup_{i=1}^N \Big\{\omega\in\Omega\colon \big|\cE_d(f_i)-\widehat{\cE}_{d,m}(f_i)(\omega) \big| \ge \varepsilon/4 \Big\} .
\end{split}
\end{equation}
Observe that for fixed
$f\in\cH$
it holds that the random variables 
$E_i:=\big(f(X^{(i)}_d)-Y^{(i)}_d\big)^2$,
$i\in\{1,2,\dots,m\}$,
are independent and satisfy
\begin{equation*}
\E\left[E_i\right]=\cE_d(f),\quad \tfrac1m \sum_{i=1}^m  E_i =\widehat{\cE}_{d,m}(f),\quad \text{and} \quad 0\le E_i \le 4\bd^2
\end{equation*}
which by Hoeffding's inequality (see~\cite[Theorem 2]{hoeffding1963})
ensures that 
\begin{equation*} 
\P\left[\big|\cE_{d}(f)-\widehat{\cE}_{d,m}(f)\big|\ge \varepsilon/4 \right]\le 2\exp\left(-\frac{m\varepsilon^2}{128\bd^4}\right).
\end{equation*}
Together with~\eqref{eq:subsets}, the monotonicity and subadditivity of the probability measure, and the measurability assumptions according to Lemma~\ref{lem:emp_meas} this implies that
\begin{equation*}
\begin{split}
\P\left[\cE_d\big(\widehat f_{d,m,\mathcal{H}}\big)-\cE_d\big(f_{d,\mathcal{H}}\big) \ge \varepsilon\right] &\le \sum_{i=1}^N \P\left[\big|\cE_d(f_i)-\widehat{\cE}_{d,m}(f_i) \big| \ge \varepsilon/4 \right] \le 2N\exp\left(-\frac{m\varepsilon^2}{128\bd^4}\right).
\end{split}
\end{equation*}
Using the complement rule and plugging in the definition of
$N$
proves the theorem.
\end{proof}

\subsection{Clipped neural networks are standard neural networks} \label{app:nn} 
We show that \enquote{clipped} neural network hypothesis classes $\cN_{\ba,R,\bd}^{u,v}$ are in fact subsets of \enquote{non-clipped} ones.
\begin{lemma}[clipping function as neural network]\label{lem:clipisNN}
Assume Setting~\ref{set:NN} and let
\begin{equation*}
\ptheta := \left(\left(\begin{bmatrix}\phantom{-}1 \\ -1\end{bmatrix},\begin{bmatrix} 0  \\ 0 \end{bmatrix}\right), \left(\begin{bmatrix}-1 & \phantom{-}0 \\ \phantom{-}0 & -1\end{bmatrix},\begin{bmatrix}\bd  \\ \bd\end{bmatrix}\right),\left(\begin{bmatrix}-1 & 1 \end{bmatrix},0\right)\right)\in\p_{(1,2,2,1),D}.
\end{equation*} 
Then it holds that
$$
\cC_{\bd}=\F(\ptheta).
$$
\end{lemma}
\begin{proof}[Proof of Lemma~\ref{lem:clipisNN}]
A case distinction establishes that for every $x\in\R$ it holds that
\begin{equation*}
  \F(\ptheta)(x)=-\ReLU_*\big(\bd-\ReLU_*(x)\big)+\ReLU_*\big(-\ReLU_*(-x)+\bd\big)=\cC_{\bd}(x),
\end{equation*}
which proves the claim.
\end{proof}
\begin{corollary}[clipped neural networks are standard neural networks] \label{cor:nnissub}
Assume Setting~\ref{set:NN}. Let $d\in\N$, $R\in [D,\infty)$, and
$\mathbf{a}=(a_0,a_1,\dots ,a_L)\in \bA_d$, and define 
\begin{equation*}
    \mathbf{b}:=(a_0,a_1,\dots ,a_{L},2,2,1)\in \bA_d.
\end{equation*}
Then it holds that
$$
\cN_{\ba,R,\bd}^{u,v}\subseteq \cN_{\mathbf{b},R}^{u,v}.
$$
\end{corollary}
\begin{proof}[Proof of Corollary~\ref{cor:nnissub}]
The proof follows by the representation of the clipping function in Lemma~\ref{lem:clipisNN} and the fact that composition with a neural network does not change the magnitude of its parameters;\footnote{Because of that we did not choose the easier representation $\cC_{\bd}(x)=-\bd+\ReLU_*\big(2\bd-\ReLU_*(\bd-x)\big)$.} see also~\cite[Definition 2.2]{petersen2017optimal} for a formal definition.
\end{proof}
\subsection{Lipschitz continuity of the realization map} \label{app:cov}
\begin{proof}[Proof of Theorem~\ref{thm:nn_lip}]
Define $L:=\cL(\ba)$ and $\mathfrak{m}:=\max \big\{1,|u|,|v|\big\}$. We will show the following stronger statement. For every
$\ptheta,\peta \in \p_{\ba,R}$
it holds that
\begin{equation} \label{eq:claim_lip}
\big\| \F^{u,v}(\ptheta) -  \F^{u,v}(\peta) \big\|_{\leb^\infty} \le   \Big[  \mathfrak{m}L R^{L-1}  \|\ba\|^L_\infty  + \sum_{l=1}^{L} l (R \|\ba\|_\infty)^{l-1} \Big] \|\ptheta -\peta \|_\infty.
\end{equation}  
This directly implies the statement of Theorem~\ref{thm:nn_lip}, as it holds that 
\begin{equation*}
    \sum_{l=1}^{L} l (R \|\ba\|_\infty)^{l-1} \le L^2 (R \|\ba\|_\infty)^{L-1} \le \mathfrak{m} L^2 R^{L-1} \|\ba\|_\infty^L.
\end{equation*}
For the proof of~\eqref{eq:claim_lip} let us fix
$\ptheta,\peta \in \p_{\ba,R}$
given by
\begin{equation*}
\ptheta=((W_l,B_l))_{l=1}^{L} \quad \text{and} \quad \peta=((V_l,A_l))_{l=1}^{L}. 
\end{equation*}
Let $\rr:=\|\ptheta -\peta \|_\infty$, and for every 
$s\in\{1,\dots,L\}$ define
the partial parametrizations 
\begin{equation*}
\ptheta(s)=((W_l,B_l))_{l=1}^{s}, 
\quad \text{and} \quad 
\peta(s)=((V_l,A_l))_{l=1}^{s},
\end{equation*}
the partial realization functions $\f_s:=\F^{u,v}(\ptheta(s))$ and $\g_s:=\F^{u,v}(\peta(s))$,
the partial errors\footnote{For vector-valued functions $f\in \cont([u,v]^d,\R^n)$ we define the uniform norm on $[u,v]^d$ by $\|f\|_{\leb^{\infty}}=\|f\|_{\leb^{\infty}([u,v]^d)}:=\max_{x\in [u,v]^d} \|f(x)\|_\infty.$}
\begin{equation*}
\mathfrak{e}_0:=0 \quad \text{and} \quad \mathfrak{e}_s:=
\big\| \f_s -  \g_s \big\|_{\leb^\infty},
\end{equation*}
and the partial maxima
\begin{equation*}
\mathfrak{m}_0:=\mathfrak{m} = \max \big\{1,|u|,|v|\big\} \quad \text{and} \quad \mathfrak{m}_s:=\max\left\{1, \big\| \f_s \big\|_{\leb^\infty} ,  \big\| \g_s \big\|_{\leb^\infty} \right\}.
\end{equation*}
We are interested in estimating the error 
$\mathfrak{e}_{L}$ and try to bound 
$\mathfrak{m}_{s}$ 
relative to
$\mathfrak{m}_{s-1}$
and 
$\mathfrak{e}_{s}$
relative to
$\mathfrak{e}_{s-1}$.
Note that for every $s\in\{2,\dots,L\}$ it holds that
\begin{equation*}
\begin{split}
\big\| \f_{s} \big\|_{\leb^\infty} &=  \big\|  W_s  \ReLU_*\big(\f_{s-1}\big)  + B_s \big\|_{\leb^\infty}
\le  R\|\ba\|_\infty  \mathfrak{m}_{s-1} + R.
\end{split}
\end{equation*}
Analogous computations for the case 
$s=1$ and the function $g_s$
establish that for every 
$s\in\{1,\dots,L\}$
it holds that
$ 
\mathfrak{m}_{s}\le R \|\ba\|_\infty \mathfrak{m}_{s-1} + R
$. 
By induction this implies that 
\begin{equation}
\label{eq:max_nn}
\mathfrak{m}_{s}\le \mathfrak{m} (R \|\ba\|_\infty)^{s} + R \sum_{l=0}^{s-1} (R \|\ba\|_\infty)^l
\end{equation}
for every $s\in\{1,\dots,L\}$.
Moreover, note that for every 
$s\in\{2,\dots,L\}$
it holds that
\begin{equation*} 
\begin{split}
\mathfrak{e}_s &= \big\| \big[ W_s  \ReLU_*\big(\f_{s-1}\big)  + B_s \big]- \big[ V_s \ReLU_*\big(\g_{s-1}\big) + A_s\big] \big\|_{\leb^\infty} \\
&\le \big\|\big[W_s-V_s\big]\ReLU_*\big(\f_{s-1}\big)\big\|_{\leb^\infty}  
+ \big\|V_s \big[\ReLU_*\big(\f_{s-1}\big)-\ReLU_*\big(\g_{s-1}\big) \big]\big\|_{\leb^\infty} + \rr \\
&\le  \|\ba\|_\infty \big( \rr \mathfrak{m}_{s-1} +R \mathfrak{e}_{s-1} \big)+ \rr.
\end{split}
\end{equation*}
Together with~\eqref{eq:max_nn}, one proves by induction that 
for every $s\in\{1,2,\dots,L\}$ it holds that
\begin{equation*} 
\mathfrak{e}_{s}\le  \Big[  \mathfrak{m} s R^{s-1} \|\ba\|^s_\infty  + \sum_{l=1}^{s} l (R \|\ba\|_\infty)^{l-1} \Big] \rr.
\end{equation*}
Setting $s=L$ proves the claim in~\eqref{eq:claim_lip}.
\end{proof}

\subsection{Covering numbers of neural network hypothesis classes}
\label{app:cov_nn_hyp}
\begin{proof}[Proof of Proposition~\ref{prop:cov_nn}]
To simplify the notation we define $\mathfrak{m}:=\max\big\{1,|u|,|v|\big\}$,
\begin{equation*}
\Delta:= \frac{\rr}{2\mathfrak{m}\cL(\ba)^2 R^{\cL(\ba)-1}\|\ba\|_\infty^{\cL(\ba)}},\quad \text{and} \quad N:=\cov\big(\p_{\ba,R},\|{\cdot}\|_\infty,\Delta\big).
\end{equation*}
Choose
$
\ptheta_1,\ptheta_2,\dots,\ptheta_N \in \p_{\ba,R}
$
such that 
for every
$\ptheta\in \p_{\ba,R}$
there exists 
$i\in\{1,2,\dots,N\}$
with
$ 
\|\ptheta-\ptheta_i\|_\infty \le \Delta
$,
which by Theorem~\ref{thm:nn_lip} and
the fact that $\cC_{\bd}$ is nonexpansive 
implies that
\begin{equation*}
\begin{split}
\big\| \cC_{\bd} \circ \F^{u,v}(\ptheta) -  \cC_{\bd} \circ \F^{u,v}(\ptheta_i) \big\|_{\leb^\infty} &\le  \big\| \F^{u,v}(\ptheta) -  \F^{u,v}(\ptheta_i) \big\|_{\leb^\infty} \\
&\le 2\mathfrak{m} \cL(\ba)^2 R^{\cL(\ba)-1}\|\ba\|_\infty^{\cL(\ba)} \|\ptheta -\ptheta_i \|_\infty  \le \rr.
\end{split}
\end{equation*}
Lemma~\ref{lem:cov_finite} and identifying $\p_{\ba,R} \simeq \big\{ \ptheta \in \R^{\sz(\ba)}: \|\ptheta\|_\infty \le R \big\}$ hence show that 
\begin{equation*}
\begin{split}
\ln \cov\big(\cN_{\ba,R,\bd}^{u,v},\|{\cdot}\|_{\leb^\infty},\rr\big) &\le 
\ln \cov\big(\cN_{\ba,R}^{u,v},\|{\cdot}\|_{\leb^\infty},\rr\big)\le \ln \cov\big(\p_{\ba,R},\|{\cdot}\|_\infty,\Delta\big) \\
&\le \sz(\ba) \ln\Big(\left\lceil\frac{R}{\Delta} \right\rceil\Big) \le \sz(\ba) \ln\Big(\frac{2R}{\Delta}\Big),
\end{split}
\end{equation*}
and this proves the proposition.
\end{proof}

\subsection{Kolmogorov equation as learning problem}
\label{app:learn}

\begin{proof}[Proof of Lemma~\ref{lem:kolm_reg}]
The proof is based on the Feynman--Kac formula for viscosity solutions of Kolmogorov equations, which states that for every 
$x\in\R^d$
it holds that
\begin{equation}\label{eq:FK_vis}
F_d(T,x)=\E\big[\varphi_d\left(S_T^x\right)\big];
\end{equation}
see~\cite[Corollary 2.23(ii)]{grohs2018approx}.
We claim that for every 
$A\in\mathcal{B}([u,v]^d)$
it holds that
\begin{equation*} 
\E\big[ \1_A(X_d) \varphi_d(S^{X_d}_T)\big]=\int_{A} \E\left[\varphi_d(S_T^x)\right] \, d\PX(x).
\end{equation*}
This would prove the lemma as it implies that for 
$\PX$-a.s.\@ 
$x\in [u,v]^d$
it holds that 
\begin{equation*}
\E\big[\varphi_d(S^{X_d}_T)\big|X_d=x\big]=\E\big[\varphi_d\left(S_T^x\right)\big],
\end{equation*}
which by~\eqref{eq:FK_vis} and Setting~\ref{set:PDE} ensures that for 
a.e.\@ 
$x\in[u,v]^d$ 
it holds that
\begin{equation*}
f^*_{d}(x)=\E\left[Y_d\big|X_d=x\right]=\E\big[\varphi_d(S^{X_d}_T)\big|X_d=x\big]=\E\big[\varphi_d\left(S_T^x\right)\big]=F_d(T,x).
\end{equation*}
For the proof of the claim let us fix 
$A\in\mathcal{B}([u,v]^d)$ and let
$g_\varepsilon\in \cont^\infty(\R^d,\R)$,
$\varepsilon\in(0,1)$,
be a family of mollifiers. For every
$\varepsilon\in(0,1)$ we
define the convolution with the indicator function 
$\1_{A,\varepsilon}:=\1_A * g_\varepsilon\in \cont^\infty(\R^d,\R)$
and the continuous and bounded mapping 
\begin{equation*}
\Phi_\varepsilon :\left\{\begin{array}{ccc}\cont([0,T],\R^d)&\to &\R \\ f&\mapsto & \1_{A,\varepsilon}(f(0)) \varphi_d(f(T)). \end{array}\right.
\end{equation*}
In~\cite[Lemma 2.6(v)]{beckbecker2018} it is shown that for every $\varepsilon\in(0,1)$ it holds that 
\begin{equation*}
\E\big[ \1_{A,\varepsilon}(X_d) \varphi_d(S^{X_d}_T)\big]=\tfrac{1}{(v-u)^d} \int_{[u,v]^d} \1_{A,\varepsilon}(x)\E\big[\varphi_d\left(S_T^x\right)\big] \, dx.
\end{equation*}
The fact that $\lim_{\varepsilon\to 0} \1_{A,\varepsilon}(x)=\1_A(x)$ for a.e.\@ 
$x\in\R^d$ (see, for instance,~\cite[Appendix C.5]{evans2010partial}) and the dominated convergence theorem prove the claim when letting $\varepsilon$ tend to zero.
\end{proof}

\subsection{Neural network approximation result for solutions of Kolmogorov equations} \label{app:approx}
The proof of Theorem~\ref{thm:mainapprox} is given after the following two auxiliary lemmas. First, we show that given an SDE with affine coefficients $\sigma_d$ and $\mu_d$, its solution $S^x_T$ also admits a (random) affine representation.
\begin{lemma}[representation of SDE solutions]\label{lem:affine}
Assume Setting \ref{set:PDE}. Let $d\in \N$, let $e_i\in\R^d$, $i\in\{1,\dots,d\}$, be the standard basis in $\R^d$, for every $p,z\in[0,\infty)$ let
\begin{equation*}
    \mathfrak{c}_p(z):= 2^{p/2}\big(z+KT+\max\{2,p\}K\sqrt{T}\big)^p\exp\big(pK^2T\big[\sqrt{T}+\max\{2,p \}\big]^2\big),
\end{equation*}
and define the random variables 
$\mathfrak{M}:\Omega\to \R^{d\times d}$ and
$\mathfrak{N}:\Omega\to \R^d$ 
by
\begin{equation*}
\mathfrak{M}:=\begin{bmatrix} S_T^{e_1}-S_T^{\zero} & S_T^{e_2}-S_T^{\zero} &\dots & S_T^{e_d}-S_T^{\zero} \end{bmatrix} \quad \text{and} \quad \mathfrak{N}:=S_T^{\zero}.
\end{equation*}
Then for every $x\in \R^d$ it holds $\P\text{-a.s.}$ that $
S^{x}_T=\mathfrak{M}x+\mathfrak{N}=\cA_{\mathfrak{M},\mathfrak{N}}(x)$
and it holds that\footnote{Recall that for a matrix $M\in\R^{d\times d}$ we denote by $\|M\|_2:=\big(\sum_{i,j=1}^d M^2_{ij}\big)^{1/2}$ its Frobenius norm.}
\begin{enumerate}[label=(\roman*)]
\item \label{it:repr1}
$\mathbb{E}\big[\left\|\mathfrak{M}\right\|_{\eu} + \left\|\mathfrak{N} \right\|_{\eu}\big]
\le 3 \mathfrak{c}_1(1)d
$ and
\item \label{it:repr2} $\big\|\mathbb{E}\big[\left\|\cA_{\mathfrak{M},\mathfrak{N}}\right\|^\nu_{\eu}\big]\big\|_{\leb^2(\PX)} \le  \mathfrak{c}_\nu(
\max\{1,|u|,|v|\})d^{\nu/2}$.
\end{enumerate}
\end{lemma}
\begin{proof}[Proof of Lemma~\ref{lem:affine}]
A proof of the first claim can be found in~\cite[Lemmas~2.7 and~2.15]{grohs2018approx}. For the proof of items~\ref{it:repr1} and~\ref{it:repr2} note that for every $p\in[0,\infty)$, $x\in \R^d$ it holds that
\begin{equation*}
\mathbb{E}\big[\|\cA_{\mathfrak{M},\mathfrak{N}}(x)\|^p_{\eu}\big]=\mathbb{E}\big[\|S^{x}_T\|^p_{\eu}\big] \le \left(\mathbb{E}\left[\|S^{x}_T\|_{\eu}^{\max\{2,p\}}\right]\right)^{p/\max\{2,p\}}\le 
\mathfrak{c}_p(\|x\|_{\eu});
\end{equation*}
see~\cite[Proposition 2.14]{grohs2018approx}. Together with the facts that it holds that
\begin{equation*}
\begin{split}
\mathbb{E}\big[\left\|\mathfrak{M}\right\|_{\eu} + \left\|\mathfrak{N} \right\|_{\eu}\big]
&\le \E\Big[  \left\| S_T^{\zero} \right\|_\eu+ \sum_{i=1}^d \left\| S_T^{e_i}-S_T^{\zero} \right\|_\eu \Big]  
\le 
(d+1) \mathbb{E}\left[\left\|S_T^{\zero}\right\|_{\eu}\right] + \sum_{i=1}^d \mathbb{E}\left[\big\|S_T^{e_i}\big\|_{\eu}\right]
\end{split}
\end{equation*}
and that
\begin{equation*}
\begin{split}
\big\|\mathbb{E}\big[\left\|\cA_{\mathfrak{M},\mathfrak{N}}\right\|^\nu_{\eu}\big]\big\|_{\leb^2(\PX)} \le \Big(\int_{[u,v]^d} \big[\mathfrak{c}_\nu(\|x\|_{\eu}) \big]^2 \, d\PX(x)\Big)^{1/2} 
\le \mathfrak{c}_\nu(\sqrt{d}\max\{1,|u|,|v|\})
\end{split}
\end{equation*}
this implies the desired estimates.
\end{proof}
In the next lemma we show that the average of the composition of a neural network with different affine functions can be represented by a single neural network and we bound the number and size of its parameters.
\begin{lemma}[compositions of neural networks and affine functions]\label{lem:affcomb}
Assume Setting~\ref{set:NN}. Let $d,n\in\N$, $\mathbf{b} \in \bA_d$, $\peta \in \p_\mathbf{b}$, and
\begin{equation*} 
((M^{(j)},N^{(j)}))_{j=1}^n\in \big(\R^{d\times d}\times \R^{d}\big)^n.
\end{equation*} 
Then there exist $\ba\in\bA_d$ and $\ptheta \in \p_\ba$ such that it holds that
\begin{enumerate}[label=(\roman*)]
\item \label{it:affcomb1}
$
\F(\ptheta) = \frac{1}{n}\sum_{j=1}^n \F(\peta)\circ  \cA_{M^{(j)},N^{(j)}}
$,
\item \label{it:affcomb2} $ 
\sz(\ba)\le n^2\sz(\mathbf{b})
$, 
\item \label{it:affcomb3}
$
\|\ptheta\|_\infty \le \sqrt{d} \|\peta \|_\infty  \max_{j=1}^n \left(\|M^{(j)}\|_{\eu} + \|N^{(j)} \|_{\eu}   +1 \right)
$,        
\item \label{it:affcomb4}
$
\cL(\ba) = \cL(\mathbf{b})
$, and
\item \label{it:affcomb5}
$\|\ba\|_\infty = n \|\mathbf{b}\|_\infty$.
\end{enumerate}
\end{lemma}
\begin{proof}[Proof of Lemma~\ref{lem:affcomb}]
With the exception of item~\ref{it:affcomb3}
this result is proven in~\cite[Lemma 3.8]{grohs2018approx}. There it is shown that for $\peta = ((V_l,A_l))_{l=1}^L$ a suitable parametrization $\ptheta = ((W_l,B_l))_{l=1}^L$
is given by $W_L:=\begin{bmatrix} \tfrac1n V_L & \tfrac1n V_L &\dots &  \tfrac1n V_L \end{bmatrix}$, $B_L:= A_L$,
\begin{equation*}
W_1:=\begin{bmatrix} V_1M^{(1)} \\ \vdots \\ V_1M^{(n)} \end{bmatrix}, \quad
B_1:=\begin{bmatrix} V_1N^{(1)}+A_1  \\ \vdots \\  V_1N^{(n)}+A_1 \end{bmatrix}, \quad \text{and} \quad W_l:=\begin{bmatrix} 
V_l	 	 &\dots		 & 0 \\ 
\vdots	 	 &\ddots 	 	 &\vdots	\\ 
0 	 & \dots 		 &  V_l \end{bmatrix}, \quad 
B_l:=\begin{bmatrix} A_l \\  \vdots \\  A_l \end{bmatrix},
\end{equation*}
$l\in\{2,\dots,L-1\}$. Now observe that
\begin{equation*}
\|W_1\|_\infty \le \sqrt{d}  \|\peta\|_\infty \max_{j=1}^n \| M^{(j)}\|_\eu \quad \text{and} \quad \|B_1\|_\infty \le \sqrt{d}  \|\peta\|_\infty \max_{j=1}^n \big( \| N^{(j)}\|_\eu + 1 \big),
\end{equation*}
which proves the lemma.
\end{proof}
Now we are ready to prove Theorem~\ref{thm:mainapprox}.
\begin{proof}[Proof of Theorem \ref{thm:mainapprox}]
Fix 
$d\in \N$,
$\varepsilon \in (0,1)$ and define
$\mathfrak{m}:=\max\{1,|u|,|v|\}$. Let $\mathfrak{M}$, $\mathfrak{N}$, $\mathfrak{c}_1(1)$, and $\mathfrak{c}_\nu(\mathfrak{m})$ be given as in Lemma~\ref{lem:affine}, let
$((\mathfrak{M}^{(j)},\mathfrak{N}^{(j)}))_{j\in\N}$ be i.i.d.\@ random variables with $(\mathfrak{M}^{(1)},\mathfrak{N}^{(1)})\sim \left(\mathfrak{M},\mathfrak{N}\right)$, and let 
\begin{equation} \label{eq:ndef}
n \in \left[ 16D^2 \varepsilon^{-1} , 32D^2 \varepsilon^{-1} \right) \cap \N \quad \text{and} \quad \neps:=\big(8\mathfrak{c}_\nu (\mathfrak{m})\big)^{-1} d^{-\nu/2} \varepsilon^{1/2}.
\end{equation}
Define 
$
g:=\F\left(\peta_{d,\delta}\right)
$ and note that Setting~\ref{set:PDE}, Lemma~\ref{lem:affine}, and the Feynman--Kac formula~\cite[Corollary 2.23(ii)]{grohs2018approx} establish that for every $x\in\R^d$ it holds that
\begin{equation} \label{eq:prop}
    |g(x)|\le D,  \quad|\varphi_d(x)  - g (x) | \le\neps ( 1+ \|x \|^\nu_{\eu}), \quad \text{and} \quad F_d(T,x) = \E\left[\big(\varphi_d \circ \cA_{\mathfrak{M},\mathfrak{N}} \big)(x) \right].
\end{equation}
We now use techniques from~\cite[Proof of Proposition 3.4]{grohs2018approx} to show that the random variable $G\colon\Omega\to [0,\infty)$, given by
$$ G:=\varepsilon^{-1/2}\underbrace{\Big\|F_d(T,{\cdot})-\tfrac{1}{n}\sum_{j=1}^{n}g\circ \cA_{\mathfrak{M}^{(j)},\mathfrak{N}^{(j)}} \Big\|_{\leb^2(\PX)}}_{:=G_1} + \big(6\mathfrak{c}_1(1)dn\big)^{-1}\underbrace{\vphantom{\sum_{j=1}^{n}} \max_{j=1}^{n}\left(\big\|\mathfrak{M}^{(j)}\big\|_\eu + \big\|\mathfrak{N}^{(j)}\big\|_\eu \right)}_{:=G_2},$$
satisfies $\E[G]\le1$. First, note that~\eqref{eq:ndef} and~\eqref{eq:prop}, together with
Jensen's inequality, Fubini's theorem, the Bienaym\'e formula (see also~\cite[Lemma 2.3]{grohs2018approx}), and Lemma~\ref{lem:affine}, ensure that
\begin{equation*}
\begin{split}
\E[G_1]
&\le 
\Big\|\E\left[(\varphi_d  - g) \circ\cA_{\mathfrak{M},\mathfrak{N}} \right] \Big\|_{\leb^2(\PX)} + \E \Bigg[\Big\| \E\left[g\circ \cA_{\mathfrak{M},\mathfrak{N}}\right] - \tfrac{1}{n}\sum_{j=1}^{n}g\circ \cA_{\mathfrak{M}^{(j)},\mathfrak{N}^{(j)}}\Big\|_{\leb^2(\PX)}\Bigg] \\
&\le \big\|\mathbb{E}\big[\neps ( 1+ \|\cA_{\mathfrak{M},\mathfrak{N}} \|^\nu_{\eu})\big] \big\|_{\leb^2(\PX)} \\
&\quad + \Big(\E\Big[ \int_{[u,v]^d} \Big( \E\left[\big(g\circ \cA_{\mathfrak{M},\mathfrak{N}}\big)(x)\right] - \tfrac{1}{n}\sum_{j=1}^{n} \big(g\circ \cA_{\mathfrak{M}^{(j)},\mathfrak{N}^{(j)}} \big)(x)\Big)^2 \, d\PX(x) \Big]\Big)^{1/2} \\
&\le \neps \big( 1+ \big\|\mathbb{E}\big[\left\|\cA_{\mathfrak{M},\mathfrak{N}}\right\|^\nu_{\eu}\big]\big\|_{\leb^2(\PX)}\big) 
+ \Big(\int_{[u,v]^d} \mathbb{V} \Big[ \tfrac{1}{n}\sum_{j=1}^{n}\big(g\circ \cA_{\mathfrak{M}^{(j)},\mathfrak{N}^{(j)}} \big)(x)\Big] \, d\PX(x) \Big)^{1/2} \\
&\le 2 \mathfrak{c}_\nu(\mathfrak{m}) d^{\nu/2} \neps+ Dn^{-1/2} \le \tfrac{1}{2}\varepsilon^{1/2}.
\end{split}
\end{equation*}
Next, observe that Lemma~\ref{lem:affine} establishes that
\begin{equation*}
\mathbb{E}[G_2]\le \mathbb{E} \Big[
\sum_{j=1}^{n}\left(\big\|\mathfrak{M}^{(j)}\big\|_\eu + \big\|\mathfrak{N}^{(j)}\big\|_\eu \right) \Big] \le n \mathbb{E}\big[\left\|\mathfrak{M}\right\|_{\eu} + \left\|\mathfrak{N} \right\|_{\eu}\big] \le 3\mathfrak{c}_1(1)dn
\end{equation*}
which proves that
$
\E[G] = \varepsilon^{-1/2}\E[G_1]+  \big(6\mathfrak{c}_1(1)dn\big)^{-1} \E[G_2] \le 1.
$
Thus, there exists $\omega\in \Omega$ such that $G(\omega)\le 1$ (see~\cite[Proposition 3.3]{grohs2018approx}), and with 
\begin{equation*}
M^{(j)}:=\mathfrak{M}^{(j)}(\omega)  \quad \text{and} \quad  N^{(j)}:=\mathfrak{N}^{(j)}(\omega), \quad j\in \{1,\dots, n\},
\end{equation*} 
it holds that
\begin{equation}\label{eq:esterr}
\tfrac{1}{(v-u)^d}\Big\|F_d(T,{\cdot})-\tfrac{1}{n}\sum_{j=1}^{n}g\circ \cA_{M^{(j)},N^{(j)}} \Big\|^2_{\leb^2([u,v]^d)} = G_1^2(\omega)\le  G^2(\omega)\varepsilon\le \varepsilon
\end{equation}
and that
\begin{equation*}
\max_{j=1}^n \left(\|M^{(j)}\|_{\eu} + \|N^{(j)} \|_{\eu} \right)  
=G_2(\omega)\le 6\mathfrak{c}_1(1)G(\omega)dn 
\le 192D^2\mathfrak{c}_1(1)d  \varepsilon^{-1}.
\end{equation*} 
By Lemma~\ref{lem:affcomb}, our assumptions, and~\eqref{eq:ndef} 
there exist $\mathbf{a}\in \bA_d$ and $\ptheta
\in \p_\mathbf{a}$ satisfying the following:
\begin{enumerate}[label=(\roman*)]
\item \label{it:prooffirst} $\cC_{\bd} \circ \F\left(\ptheta\right) =  \F\left(\ptheta\right)
=\tfrac{1}{n}\sum_{j=1}^{n}\F(\peta_{d,\neps})\circ \cA_{M^{(j)},N^{(j)}} =\tfrac{1}{n}\sum_{j=1}^{n} g\circ \cA_{M^{(j)},N^{(j)}}$;
\item $ 
\sz(\ba)\le n^2\sz(\mathbf{b}_{d,\neps}) \le 32^2 D^4 \zeta d^{\gamma}\varepsilon^{-2} \neps^{-\lambda}\le  C d^{\nu\lambda/2+\gamma} \varepsilon^{-\lambda/2-2}
$;
\item $ 
\|\ptheta\|_\infty\le 
\sqrt{d} \|\peta_{d,\neps}\|_\infty \left(192 D^2 \mathfrak{c}_1(1)d \varepsilon^{-1} +1\right) \le Cd^{(\nu\kappa+3)/2+\beta}\epsilon^{-\kappa/2-1}  
$;
\item $
\cL(\ba) = \cL(\mathbf{b}_{d,\neps})= \cL(\mathbf{b}_{d,cd^{-\nu/2}\varepsilon^{1/2}})
$; and
\item \label{it:prooflast} $ 
\|\ba\|_\infty = n \|\mathbf{b}_{d,\neps}\|_\infty \le 32D^2 \varepsilon^{-1} \|\mathbf{b}_{d,\neps}\|_\infty \le C \varepsilon^{-1}\|\mathbf{b}_{d,cd^{-\nu/2}\varepsilon^{1/2}}\|_\infty
$, 
\end{enumerate}
where 
$C:=\zeta\max\big\{ 32^2   D^4 \big(8 \mathfrak{c}_\nu(\mathfrak{m})\big)^\lambda, \big(192D^2 \mathfrak{c}_1(1) +1\big)\big(8\mathfrak{c}_\nu(\mathfrak{m})\big)^\kappa \big\}$ and
$c:=\big(8\mathfrak{c}_\nu(\mathfrak{m})\big)^{-1}$.
Together with~\eqref{eq:esterr} 
this proves the theorem.
\end{proof}
\section*{Acknowledgements}
The authors are grateful to Shahar Mendelson and Stefan Steinerberger for their useful comments.

\bibliographystyle{siamplain}
\bibliography{bibfile}

\end{document}